\pdfoutput=1

\documentclass[11pt]{article}

\usepackage[preprint]{acl}

\usepackage{times}
\usepackage{latexsym}

\usepackage[T1]{fontenc}

\usepackage[utf8]{inputenc}

\usepackage{microtype}

\usepackage{inconsolata}

\usepackage{graphicx}

\usepackage{booktabs}       
\usepackage{amsmath}
\usepackage{multirow} 
\usepackage{amssymb}
\usepackage{mathtools}
\usepackage{amsthm}
\usepackage{algorithm}
\usepackage{algorithmic}
\usepackage{spverbatim}
\usepackage{xcolor}
\usepackage{graphicx}

\newtheorem{theorem}{Theorem}

\newtheorem{lemma}{Lemma}

\newtheorem{definition}{Definition}

\usepackage{paralist}

\newcommand{\kg}{\mathcal{G}}
\newcommand{\entities}{\mathcal{E}}
\newcommand{\relations}{\mathcal{R}}

\newcommand{\uf}{\mathcal{P}}
\newcommand{\ukge}{\hat{P}}

\newcommand{\method}{SRC}

\makeatletter
\newcommand\incircbin
{%
  \mathpalette\@incircbin
}
\newcommand\@incircbin[2]
{%
  \mathbin%
  {%
    \ooalign{\hidewidth$#1#2$\hidewidth\crcr$#1\bigcirc$}%
  }%
}

\newcommand{\oland}{\incircbin{\land}}
\newcommand{\olor}{\incircbin{\lor}}

\usepackage{tikz}
\newcommand*\circled[1]{\tikz[baseline=(char.base)]{
            \node[shape=circle,draw,inner sep=1pt] (char) {#1};}}

\newcommand{\ozero}{\circled{0}}
\makeatother

\title{Extending Complex Logical Queries on Uncertain Knowledge Graphs}

\author{Weizhi Fei$^{*\clubsuit}$, Zihao Wang\thanks{ Equal Contribution} $^{\dagger}$,  Hang Yin$^{\clubsuit}$, Yang Duan$^\bigstar$,  Yangqiu Song$^{\dagger}$ \\
$\clubsuit$ Department of Mathematical Sciences, Tsinghua University, Beijing, China\\
$\dagger$ CSE, HKUST, HKSAR, China\\
$\bigstar$ Department of Computer Science, Princeton University, Princeton, United States \\
\texttt{\{fwz22,h-yin20\}@mails.tsinghua.edu.cn} \\
\{zwanggc,yqsong\}@cse.ust.hk, yd1202@cs.princeton.edu
}

\begin{document}
\maketitle
\begin{abstract}
The study of machine learning-based logical query answering enables reasoning with large-scale and incomplete knowledge graphs. This paper advances this area of research by addressing the uncertainty inherent in knowledge. While the uncertain nature of knowledge is widely recognized in the real world, it does not align seamlessly with the first-order logic that underpins existing studies. To bridge this gap, we explore the soft queries on uncertain knowledge, inspired by the framework of soft constraint programming. We propose a neural symbolic approach that incorporates both forward inference and backward calibration to answer soft queries on large-scale, incomplete, and uncertain knowledge graphs. Theoretical discussions demonstrate that our method avoids catastrophic cascading errors in the forward inference while maintaining the same complexity as state-of-the-art symbolic methods for complex logical queries. Empirical results validate the superior performance of our backward calibration compared to extended query embedding methods and neural symbolic approaches.
\end{abstract}

\section{Introduction}

Representing and reasoning with factual knowledge are essential functionalities of artificial intelligence systems. As a powerful way of knowledge representation, Knowledge Graphs (KGs)~\citep{Miller1995WordNetlexical,suchanek_yago_2007,vrandecic_wikidata_2014} use nodes to represent entities and edges to encode the relations between entities. Recently, Complex Query Answering (CQA) over KGs has attracted considerable attention because this task requires multi-hop logical reasoning over KGs and supports many applications~\citep{ren_neural_2023}. 
This task requires answering the existential First Order Logic (FOL) query, involving existential quantification ($\exists$), conjunction ($\land$), disjunction ($\lor$), and negation ($\lnot$). While answering FOL queries has been extensively researched by database community~\citep{riesen2010exact,hartig2007sparql}, such studies overlook the \textbf{incompleteness} of most KGs. Consequently, conventional graph traversal methods for \textit{relational database queries} may neglect certain answers due to the missing links of KGs. In recent studies on \textit{complex logical queries on knowledge graphs}, the generalizability of machine learning models is leveraged to predict the missing links of observed KGs and conduct first-order logic reasoning~\citep{ren_beta_2020,arakelyan_complex_2021,liu_neural-answering_2021,wang_logical_2023}. This combination of machine learning and logic enables further possibilities in data management~\citep{ren_neural_2023}.

\begin{figure*}[t]
    \centering
    \includegraphics[width=.8\linewidth]{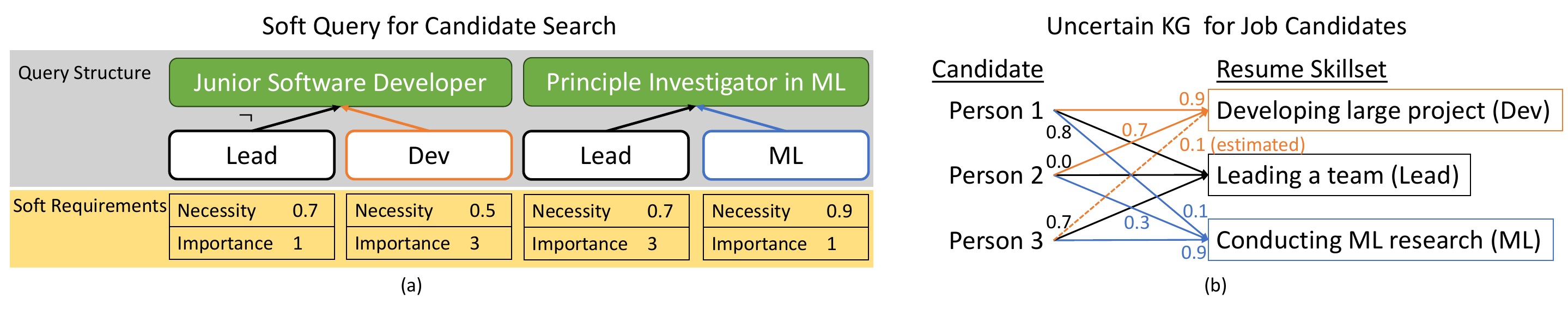}\vskip-1em
    \caption{\textbf{(a)} Examples of two soft queries in the candidate search procedure. The soft queries introduced in this paper are jointly defined by first-order logic and soft requirements. In particular, soft requirements (necessity and importance) are introduced to characterize fine-grained decision-making preferences, distinguishing them from first-order queries. \textbf{(b)} Incomplete uncertain KG for to what extent a candidate possesses a skill. Solid lines indicate the observed knowledge, while dashed lines indicate the unobserved data. Values indicate confidence level, where the higher value indicates the fact is more likely to be true.}\vskip-1em
    \label{fig:example-of-soft-query}
\end{figure*}
The uncertainty of knowledge is widely observed ranging from daily events~\citep{zhang2020aser} to complex biological systems~\citep{szklarczyk2023string}. \textit{To represent the uncertain knowledge}, confidence values $p$ are associated with facts to augment the KG~\citep{carlson2010toward,robyn2017conceptnet,szklarczyk2023string}, known as the uncertain KG. As exemplified in the right of Figure~\ref{fig:example-of-soft-query}, which illustrates the Job Candidates uncertain KG, confidence values are used to quantify the degree of practice level associated with specific skills. Uncertainty is prevalent in existing KGs primarily because most are constructed using machine learning models, where the predicted likelihood of relational facts inherently introduces uncertainty. Notable examples include O*NET~\citep{pai2021learning}, STRING~\citep{szklarczyk2023string}, ConceptNet~\citep{robyn2017conceptnet}, and AESR~\citep{zhang2020aser}.  To address the incompleteness of uncertain KGs, recent studies estimate the confidence values of missing facts with the generalization power of ML models~\citep{chen2019embedding,pai2021learning}.

\begin{table*}[t]
\caption{Comparison of different problem settings. FO: First Order, EFO: Existential First Order.}\label{tb:setting-comparison}
\centering
\resizebox{1.0\linewidth}{!}{
\begin{tabular}{llll}
\toprule
Problem Settings & Language of reasoning & Uncertainty of Knowledge  & Unobserved Knowledge \\ \midrule
Relational database & FO                                 & -                    & -                                        \\
Probablistic database & FO & Confidence value $p$ & -                                        \\
Open world probablistic database & FO & Confidence value $p$ & Uniformed $p_u$ for all unobserved facts \\
Complex logical queries on KG & EFO                     & -                    & ML generalization to unobserved facts \\
Soft queries on uncertain KG (Ours) & EFO + Soft requirements & Confidence value $p$ & ML generalization to unobserved facts and confidence $p$      \\ \bottomrule
\end{tabular}
}
\end{table*}


To reason with uncertainty, many extensions of first-order logic have been made to cope with the uncertainty in knowledge representation systems formally~\cite{adams1996primer}. It is noteworthy to mention that the study of \textit{probabilistic databases} also extends the relational databases with a confidence value $p\in[0,1]$~\cite{cavallo1987theory,suciu2022probabilistic}. From a machine learning perspective, however, previous studies in \textit{open world probabilistic databases} are limited from two aspects: (1) They assume uniform uncertainty for all unobserved knowledge~\cite{ceylan2021open}, leading to weaker characterization of the incomplete knowledge without the generalization. 
(2) They focus on first-order logic, which might be insufficient to describe practical reasoning processes with uncertainty. To address the limitations, we extend the complex query answering to uncertain KG and propose soft queries combining query structure and soft requirements, as shown in the left of Figure~\ref{fig:example-of-soft-query}.


This paper studies the machine learning method for reasoning with incomplete and uncertain knowledge, advancing previous studies in symbolic probabilistic databases. Its contribution is threefold.

\noindent\textbf{Contribution 1: A novel and practical setting.}  We propose a novel setting of Soft Queries on Uncertain KG (SQUK). Our setting extends the previous setting of \textit{complex logical queries on KGs} in two ways:
(1) For the incomplete knowledge base, SQUK extends the incomplete KG to incomplete and \textit{uncertain} KG. (2) For the language describing reasoning, SQUK extends first-order language to \textit{uncertainty-aware} soft queries with soft requirements, which are motivated by real-world reasoning with uncertainty and the establishment in soft constraint programming~\citep{schiex1992possibilistic,rossi_handbook_2006}. 
We also introduce the formal definition in Section~\ref{sec:formulation}. The comparison of SQUK against other settings is detailed in Table~\ref{tb:setting-comparison}. 

\noindent\textbf{Contribution 2: ML method for soft queries.} We bridge machine learning and SQUK by proposing Soft Reasoning with calibrated Confidence values (\method), which uses Uncertain Knowledge Graph Embeddings (UKGEs) to tackle the unobserved information and achieves the same computational complexity as the state-of-the-art inference algorithms~\citep{bai_answering_2023,yin2024rethinking}. The error analysis is also conducted for \method~ given the error bound function, characterizing how the performance is affected by UKGEs and the query structures. Based on our analysis, we suggest calibrating the confidence by debiasing and learning, which further boosts the performance of \method.

\noindent\textbf{Contribution 3: Extensive empirical studies.} We also conduct extensive empirical studies to benchmark the performance of a broad spectrum of methods under the UKGE~\citep{chen2019embedding} settings. The calibrated \method~is compared against baselines including  
Query Embedding ~\citep{ren_beta_2020} methods with Number Embeddings (QE+NE)~\citep{vaswani_attenion_2017} and symbolic search method~\citep{yin2024rethinking}. In particular, we compared the differences between QE+NE and \method~under various soft query settings, demonstrating the advantage of \method. We also make a fair comparison with large language models on annotated soft queries in a natural language setting. 

We highlight the uniqueness of the SQUK setting by examining the differences between uncertain KGs and KGs, comparing soft queries with logical queries, and addressing the challenges posed by two versions of incomplete knowledge. These differences are also summarized in Table~\ref{tb:setting-comparison}. Additionally, we present the related work concerning complex logical query answering and uncertain knowledge graph embedding in Appendix~\ref{related work}.
\section{Background}\label{background}



Uncertain KGs enhance traditional KGs by augmenting each triple fact with a confidence value, thereby facilitating the modeling of uncertain knowledge, which is particularly useful in various domains. Figure~\ref{fig:example-of-soft-query} illustrates uncertainty in job backgrounds. We formally define an uncertain knowledge graph as a set of knowledge as follows:
\begin{definition}[Uncertain knowledge graph]
    Let $\entities$ be the set of entities and $\relations$ be the set of relations, an uncertain knowledge graph $\mathcal{G}$ is a set of quadruple $\{(s_i,r_i,o_i,p_i)\}$, where $s_i,o_i\in \entities$ are entities, $r_i\in \entities$ is relation and $p_i \in [0,1]$\footnote{The values of uncertain KGs also indicate the strength or importance. For simplify, previous work ~\citep{pai2021learning} normalized the range of values into the interval $[0, 1]$.} represents the confidence value for the relation  fact $(s_i,r_i,o_i)$. This confidence value $p_i$ indicates the degree of certainty regarding the truth of the fact.
\end{definition}

Following the closed-world assumption~\citep{reiter1981closed}  and treating all unobserved facts as false, we can derive the weight graph form for uncertain KG and represent it with the confidence function $P: \entities\times\relations\times\entities\mapsto [0, 1]$ as follows:
    \begin{equation}
        P(h_i, r_i, t_i) = \left\{\begin{array}{cc}
            p_i & (h_i, r_i, t_i, p_i) \in  \mathcal{G}, \\
            0 & {\rm otherwise}.
            \end{array}\right.
    \end{equation}

Uncertain KGs also suffer incomplete issues~\citep{chen2019embedding,chen2021probabilistic}, with observed knowledge representing only a small portion of the total facts. The assumption that all unseen relational facts are false is inappropriate in real-world scenarios. To address this challenge, previous research on uncertain KGs has proposed a machine learning task to predict the confidence scores of these unseen relational facts~\citep{chen2019embedding,chen2021probabilistic}. Typically, the observed knowledge in uncertain KGs is split into three nested sets of facts,  where $\mathcal{G}_{\text{train}} \subsetneq \mathcal{G}_{\text{valid}} \subsetneq \mathcal{G}_{\text{test}}$. The training set  $\mathcal{G}_{\text{train}}$ is used to train the model, while the validation and test sets are used to evaluate its performance in predicting the confidence scores of unseen facts.


Uncertain Knowledge Graph Embeddings (UKGEs)~\cite{chen2019embedding,chen2021probabilistic} have been the mainstream methods for predicting unseen relational facts in uncertain KGs, as they learn low-dimensional representations that effectively capture the semantics between relations and facts, demonstrating strong generalizability. UKGEs are trained on partial facts $\mathcal{G}_{\text{train}}$ and approximate the confidence function $\uf$ deriving from complete facts,  defined as the following confidence function:

\begin{definition}
    An UKGE parameterizes a differentiable confidence function $\ukge: \entities \times \relations \times \entities \mapsto [0, 1]$.
\end{definition}
In practice, obtaining complete facts is challenging, so  $P_{\text{test}}$ induced by $\mathcal{G}_{\text{test}}$ are usually substituted for $\uf$, which  adhere to previous approaches~\cite{chen2019embedding,pai2021learning}. We present the connection between this setting and the open-world assumption in Appendix~\ref{app:connection with owa}.

\section{Soft Queries}
\label{sec:formulation}
The uncertainty inherent in  KGs can be modeled using confidence values for each knowledge. However, current complex logical queries are defined on a boolean basis~\footnote{For details on logical queries, please refer to Appendix~\ref{logical queries}.}, which is not compatible with uncertain KGs. This uncertainty necessitates new definitions for logical operations and answer sets. In this section, we introduce the definition of our extended soft queries.

\subsection{Syntax and semantic}

\begin{definition}[Syntax of soft queries]\label{def:soft-efo-query}
Soft queries are the disjunction of soft conjunctive queries $\phi_i$:
\begin{align}\label{eq:soft-query-dnf}
    \Phi(y) = \phi_1 (y) \olor \cdots \olor \phi_q (y),
\end{align}
where $y$ is the free variable. Each  $\phi_i(y)$ is the conjunction of the soft atomic formula:
\begin{small}
\begin{align}\label{eq:soft-conjunctive-query}
    \phi_i(y) = \exists x_1, \dots, x_{n}. a_{i1} \oland \cdots \oland a_{ij}, i=1,...,q,
\end{align}
\end{small}
where $x_1, \dots, x_{n}$ represent existentially quantified variables.  
Each $a_{i}$ is a soft atomic formula of the form $(h,r,t,\alpha, \beta)$ or its negation $\lnot (h,r,t,\alpha, \beta)$. Here, $r$ denotes the relation, $h$ and $t$ can be either an entity in $\entities$ or a variable in $\{y, x_1, ..., x_n\}$. $\alpha$ represents the necessity value, and $\beta$ represents the importance value. The $\oland$ and $\olor$ represent soft conjunction and disjunction operation. 
\end{definition}

\begin{definition}[Substitution]
    For a soft query involving variables, the substitution replaces all occurrences of the variable $x$ (or $y$)  with any entity $s \in \entities$ simultaneously, denoted as $s/x$ (or $s/y$).
\end{definition}
We denote $\phi(s)$ for the result of substituting $s$ for the free variable $y$. When all variables in the soft query $\phi$ have been substituted, we refer to it as the substituted query. Next, we define the semantics of the soft queries, starting with the soft atomic formula. Specially, the soft atomic formula involves two novelty concepts: $\alpha$ necessity and $\beta$ importance, which are inspired form soft Constraint Satisfaction Problems (CSPs)~\cite{rossi_handbook_2006} to manipulate the uncertainty of facts. We introduce the related work of soft CSPs in Appendix~\ref{app:semiring-CSP}.

1. The $\alpha$ necessity component draws inspiration from possibilistic CSPs~\citep{schiex1992possibilistic} and is designed to capture \textbf{necessity criteria}. It serves the purpose of filtering out unnecessary constraints and involves a thresholding operation.
The thresholding operation $[p]_{\alpha}$ is defined as follows:
        \begin{align}
            [p]_\alpha = \left\{\begin{array}{cc}
            p & p \ge  \alpha, \\
            \ozero & {\rm otherwise}.
            \end{array}\right.
        \end{align}
        
2. The $\beta$ importance component is influenced by weighted CSPs~\citep{bistarelli1999semiring} to describe \textbf{preference}, which is the weight employed to adjust the relative significance of different conditions.



\begin{definition}[Semantic of soft queries]
Given a semiring $(\mathbb{R}^{+}, \oplus, \otimes, \ozero)$ over $\mathbb{R}^{+}$, the confidence function $P$ induced by an uncertain knowledge graph $\mathcal{G}$, and a soft  query $\phi$, let $s$ and $o$ be entities in $\entities$. The confidence value $U(\phi, P)$ is recursively defined as follows:
    \begin{compactenum}
        \item If $\phi$ is the substituted soft atomic query $(s,r,o,\alpha, \beta)$, then $U(\phi, P) =  \beta[P(s,r,o)]_{\alpha}$;
        \item If $\phi$ is the negation of the substituted soft atomic $\lnot (s,r,o,\alpha, \beta)$, then $U(\phi, P) =  \beta[1 - P(s,r,o)]_{\alpha}$;
        \item If $\phi = \exists x_i \psi(y; x_i)$ is the soft query involving existentially quantified variables, then $U(\phi, P) = \oplus_{s\in\entities} U(\phi(y;s/x_i), P)$;
        \item If $\phi$ is the conjunctive query $(\phi_1 \oland \phi_2)$, then  $U(\phi, P) = U(\phi_1, P) \otimes U(\phi_2, P)$.
        \item If $\Phi$ is the disjunctive query $(\Phi_1 \olor \Phi_2)$, then $U(\Phi, P) = U(\Phi_1, P) \oplus U(\Phi_2, P)$.
    \end{compactenum}
\end{definition}

To align with the semantics of the confidence value, we instantiate the semiring as $(\otimes,\oplus,\ozero) = (+,\max,-\infty)$.  We discuss the utilized semiring of the previous soft CSP setting in Appendix~\ref{app:semiring-CSP}. With semantics, we can compute the utility of entity $s$ for the soft query $\phi$. The utility of all entities can be conveniently represented as a vector.

\begin{definition}[Utility vector]\label{def:utility vector}
    Given a confidence function $P$ induced by an uncertain knowledge base and a soft query $\phi$, the utility vector of soft query $\Phi$, denoted as $\mathbf{u} \in \mathbb{R}^{|\entities|}$, is defined as:
    \begin{equation}
        \mathbf{u}_i = U(\Phi(s_i/y), P),
    \end{equation}
    where $s_i$ denotes the entity indexed by $i$.
\end{definition}

\subsection{Example to explain the soft queries, as well as the necessity  and importance}

Soft queries are a powerful tool for modeling candidate searches across various job positions, enabling nuanced assessments of qualifications. To illustrate this, we present an example of using soft queries to model candidate searches for two roles: Junior Software Developer (JSD) and Principal Investigator (PI) in machine learning. Let \textsc{Has} denote the relation describing a candidate’s possession of a skill, while \textsc{LEAD}, \textsc{DEV}, and \textsc{ML} represent leadership, development, and machine learning skills, respectively. Regarding leadership, the Principal Investigator places a significantly high emphasis on it, while the Junior Software Developer does not. This distinction is captured by the importance parameter $\beta$, which assigns greater weight to leadership in the query for the Principal Investigator. Additionally, the necessity parameter $\alpha$ serves as a threshold to filter out candidates who lack the required skills. The two roles can be modeled using the following soft queries, as further explained in Figure~\ref{fig:example-of-soft-query}:

\begin{small}
\begin{align*}
\phi_{\rm JSD}(y) = & \lnot (y, \textsc{Has}, \textsc{Lead}, 0.7, 1)) \oland (y,\textsc{Has}, \textsc{Dev}, 0.5, 3),\\
\phi_{\rm PI}(y) = & (y, \textsc{Has}, \textsc{Lead}, 0.7, 3) \oland (y,\textsc{Has}, \textsc{ML}, 0.9, 1)).
\end{align*}
\end{small}

Beyond recruitment, soft queries provide an effective and tailored strategy for modeling logical queries on uncertain KGs. This capability can extend the application of logical queries to uncertain KGs, enabling use cases such as product recommendation~\citep{bai2024advancing} and understanding user intentions~\citep{bai2024understanding}.


\subsection{Soft query graph and utility vector}
\begin{definition}[Soft query graph]
Given a soft query $\phi(y;x_1, ..., x_n) = \exists x_1, ..., x_n. a_1 \oland \cdots \oland a_{m}$,
the soft query graph $G_\phi$ is defined by tuples induced by soft atomic formulas or its negation: $G_\phi = \{ (h_i, r_i, t_i, \alpha_i, \beta_i, \textsc{Neg}_i) \}_{i=1}^{m}$,
where $(h_i, r_i, t_i, \alpha_i, \beta_i, \textsc{Neg}_i))$ is induced by $ (h_i,r_i,t_i,\alpha_i,\beta_i)$ or $ \lnot ( h_i,r_i,t_i, \alpha_i,\beta_i)$. $\textsc{Neg}_i$ is the bool variable indicating if  $a_i$ is negated.
\end{definition}

If $h$ (or $t$) is a variable, we say the corresponding node in $G_\phi$ is a \textit{variable node}. $V(G_\phi)$ indicates the set of all variable nodes in $G_\phi$.
If $h$ (or $t$) is an entity, we say the corresponding node is an \textbf{constant node}. The \textbf{leaf node} is a node that is only connected to one other node in the query graph. Compared to the operation tree in complex logical queries~\citep{ren_beta_2020}, the soft query graph can model any conjunctive soft queries.

\section{Methodology}\label{sec:method}


In this section, we propose Soft Reasoning with calibrated Confidence values (\method) to facilitate reasoning with various query structures and soft requirements. \method ~is a symbolic reasoning method that utilizes UKGE to provide confidence values. Since UKGE inevitably has prediction errors, we present a mild assumption regarding the UKGE error bound in Equation~\eqref{eq:ukge-error-assumption}. Our error analysis over \method~indicates that the inference error is manageable as the complexity of the query structure increases. To further reduce this error, we introduce two orthogonal calibration strategies: \textit{Debiasing} (D) and \textit{Learning} (L).

\subsection{Forward inference}\label{sec:inference}
\begin{figure*}[h]
  \centering
  \includegraphics[width=.95\textwidth]{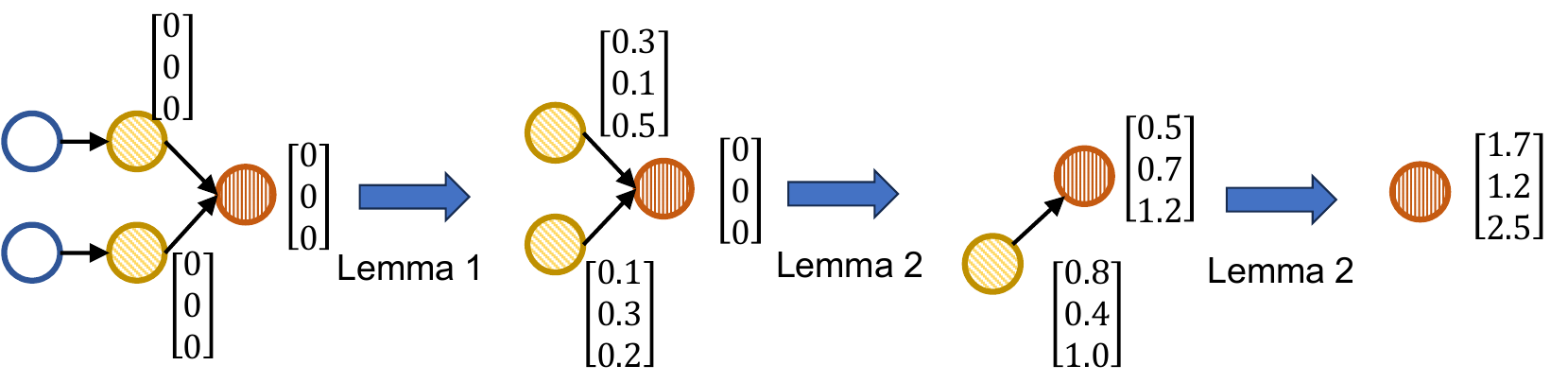}
  \caption{A toy model illustrating the process of the \method. Each variable node is assigned a state vector, which is updated through the algorithm as edges are removed. The final state vector of the free variable is the desired.}
  \label{fig: algorithm}
  \vspace{-1.5em}

\end{figure*}

The main paper discusses soft queries in which the query graphs are acyclic simple graphs. Cases with the complete case (cycles and self-loops) are detailed in Appendix~\ref{details methods}.

Given the soft query $\phi$, \method~ efficiently derives the utility vector $U(\phi, \ukge)$ based on the confidence function $\ukge$ approximated by UKGE.  The core idea is to progressively prune the edges of the soft query graph while preserving the constraints of the remaining edges, ensuring that the final utility vector remains unchanged. State vectors are used to record the constraints of the pruned edges during the inference process. Specifically, each variable node $z\in G_\phi$ is described by a \textit{state vector} $C_z \in \mathbb{R}^{|\mathcal{E}|}$. The notation $(G_\phi, \{C_z: z\in V(G_\phi)\})$ denotes a soft query graph with state vectors. 

We define equivalent transformations as $T$: 
{\scriptsize
\begin{equation}
    T(G_{\phi}, \{C_z: z\in V(G_\phi)\}) = (G_{\psi},\{C_z':  z\in V(G_\psi)\}),
\end{equation}}
where $G_\psi$ is a subgraph of $G_\phi$ (with at least one edge eliminated), $C_z'$ is the updated state vector, and $T$ guarantees the utility vector $\hat{\mathbf{u}}$ unchanged.


Two lemmas are presented to induce two equivalent transformations, denoted as $T_e$ and $T_l$, respectively. The proof can refer to Appendix \ref{details methods}.
\begin{lemma}\label{Remove_constant_nodes}
For each constant node in $G_\phi$, an $O(|\mathcal{E}|)$ transformation $T_c$ exists to remove it. 
\end{lemma}
Realization of the equivalent transformation $T_c$ induces a function $\textsc{RemoveConstNode}$.
\begin{lemma}\label{Remove_leaf_nodes}
 For each leaf node in $G_\phi$, an $O(|\mathcal{E}|^2)$ transformation $T_l$ exists to remove it. 
\end{lemma}
Realization of the equivalent transformation $T_l$ induces a function $\textsc{RemoveLeafNode}$.

\begin{algorithm}[t]
\caption{SRC (simple acyclic case)}\label{alg:srcc}
\begin{algorithmic}
\REQUIRE Input soft query graph ${G}_{\phi}$ and initialize the state vectors $\{ C_z \}$.
\ENSURE Output utility vector $\hat{\mathbf{u}}(G_{\phi}, \{ C_z \})$.
 \STATE {\small $( G_{\phi},\{ C_z\}) \gets \textsc{RemoveConstNode}( G_{\phi}, \{ C_z\})$ }
     \WHILE{There exists a leaf node} 
      \STATE {\small $(G_{\phi}, \{ C_z\}) \gets \textsc{RemoveLeafNode}(G_{\phi}, \{ C_z\})$ }
      \ENDWHILE
 \STATE Get the utility vector by retrieving $C_y$.
\end{algorithmic}
\end{algorithm}
\label{details methods}
The above two transformations are constructive and can be applied to remove edges once the corresponding nodes are found.  Until the soft query graph only contains a free variable node, the state vector of the free variable $C_y$ is the desired utility vector. The procedure of SRC for acyclic soft query graphs is presented in Algorithm~\ref{alg:srcc} and one toy example of the execution is visualized in Figure~\ref{fig: algorithm}. We first remove constant nodes, as these are commonly found and can be easily eliminated using Lemma~\ref{Remove_constant_nodes}. Subsequently, we identify and remove leaf nodes step by step, according to the Lemma~\ref{Remove_leaf_nodes}, noting that the leaf node always exists for acyclic queries.


\noindent\textbf{Complexity analysis.}  We begin with a rough estimation of the complexity. The space complexity of the inference algorithm is $O(|\mathcal{R}| |\mathcal{E}|^2)$. Let $n_e$ denote the number of edges involving existential variables and $n_r$ represent the remaining edges. For acyclic queries, the time complexity is $O(n_e |\mathcal{E}|^2 + n_r |\mathcal{E}|)$. We further reduce complexity by leveraging sparsity, a natural characteristic of knowledge graphs. For space complexity, only non-zero values or those exceeding a threshold $\delta_1$ are stored and the average sparsity ratio of our constructed neural matrices is approximately $3.4\%$, reducing storage requirements by $97\%$. For time complexity, sparsity can be used to accelerate computation. Removing a leaf node $x$ typically involves $O(E^2)$ operations, but by considering only rows where $C_x$ is non-zero, this reduces to $O(\mathcal{E} \cdot (|C_x > \delta_2|))$, where $|C_x > \delta_2|$ represents the non-zero number of $C_x$. Thus, the  time complexity becomes $O(n_e \mathcal{E} \cdot \max_x(|C_x > \delta_2|))$, where $n_e$ is the number of edges involving existential variables.

\subsection{Error analysis}\label{sec:error-analysis}
The current model still exhibits significant prediction errors, with the mean absolute error on CN15K generally reaching around 0.2~\citep{chen2021probabilistic,chen2021passleaf}. To facilitate the error analysis of our proposed search algorithm, we introduce the error bound $\varepsilon(\delta)$ as the following:
\begin{definition}
    Let $\varepsilon$ be a function that maps the error $\delta$ to a tail probability which can uniformly bounds the error of $\ukge$ for some kind of norm:
    {\small
    \begin{align}\label{eq:ukge-error-assumption}
        \Pr\left(\max_{(s,r,o)}\| \ukge(s,r,o) - \uf(s, r, o) \| > \delta\right) < \varepsilon(\delta).
    \end{align}}
\end{definition}
We note that $\varepsilon(\delta) < 1$ is guaranteed for all $\delta$ and $\varepsilon(\delta)$ decreases monotonically with $\delta$, though a better link predictor provides a tighter $\varepsilon(\delta)$. This definition provides a practical and flexible approach for analyzing the errors over uncertain KGs. We look into the error of each soft atomic query:
\begin{theorem}\label{thm:1p-analysis}
    For any soft atomic query $\psi = (h,r,y,\alpha, \beta)$, let the uniform inference error be
    $$\max_{\psi,s\in\entities}\|U(\psi(s/y), \ukge) - U(\psi(s/y), \uf)\| =\epsilon(\alpha,\beta)$$
    Then we estimate the distribution of $\epsilon(\alpha,\beta)$ by the uniform error-bound $\varepsilon(\delta)$ provided in Equation~\eqref{eq:ukge-error-assumption} and assume the probability density function of $\uf$ is $f(\xi)$ :
    {\scriptsize
    \begin{align}\label{eq:1p-bound}
     \Pr\left(\epsilon(\alpha,\beta) > \delta\right)\nonumber
     <\varepsilon(\frac{\delta}{\beta}) +
    (1-\varepsilon(\frac{\delta}{\beta}))\int_{0}^{1}\varepsilon(|\alpha-\xi|)f(\xi){\rm d}\xi.
    \end{align}}
\end{theorem}

 Moreover, the numerical stability is guaranteed:

\begin{theorem}\label{thm:numerical-stability}
    For a soft conjunctive query $\phi = \exists x_1, ..., x_n. a_{1} \oland \cdots \oland a_{m}$, where $a_i = (h_i,r_i,t_i,\alpha_i, \beta_i)$, and any entity $s\in\entities$, the error accumulated is at most linear:
    {\small
    \begin{align}
        \|U(\phi(s), \ukge) - U(\phi(s), \uf)\|\leq \Sigma_{i=1}^{m} \epsilon(\alpha_i,\beta_i).
    \end{align}
    }
    
\end{theorem}
This conclusion ensures that there is no catastrophic cascading error in our forward inference algorithm. The proof of all the above theorems can refer to Appendix~\ref{app: error analysis}. 

\subsection{Two calibration strategies}\label{sec:calibration}

\noindent\textbf{Debiasing.} The confidence function $\ukge$ of UKGE is biased towards zero. We propose a debiasing strategy for the inference. We modify the soft requirements $\alpha$ as $\alpha - \Delta_\alpha$. We can see that this simple debiasing strategy improves performance. SRC with this strategy is denoted as SRC(D).

\noindent\textbf{Learning.}
The pre-trained UKGE is not optimal for SRC in the incomplete uncertain KGs. We propose the \textit{calibration by learning (L)} strategy by learning the calibrated confidence function. Specifically, we calibrate the confidence function $\hat P_c$ by learnable affine transformation~\citep{arakelyan_adapting_2023} as following:

\begin{small}
\begin{align}
     &\ukge_c(s,r,o) = \ukge(s,r,o) (1+ \rho_{\theta}(s,r,o)) + \lambda_{\theta}(s,r,o),\\ \notag
     & [\rho_{\theta}(s,r,o),\lambda_{\theta}(s,r,o)] = \sum_{j \in \{s,r,o\}} (W_j \mathbf{e}_j + b_j).
\end{align}
\end{small}
Here, $\rho_{\theta}(s,r,o)$ and $\lambda_{\theta}(s,r,o)$ are the affine parameters. $\mathbf{e}_j \in \mathbb{R}^{d}$ represents the embedding for entity or relation $j \in \{s,r,o\}$ with embedding dimension $d$. Meanwhile, $W_j \in \mathbb{R}^{2 \times d}$ and $b_j\in \mathbb{R}^{2}$ are the learnable parameters associated with $j$.

As implied by both Theorem~\ref{thm:1p-analysis} and Theorem~\ref{thm:numerical-stability}, the error of SRC is rooted in the error bound of UKGE. As we can see from Equation~\eqref{eq:1p-bound} the error bound is governed by both $\varepsilon$ and the integral of $\varepsilon$ over the domain $[0, \max(\alpha, 1-\alpha)]$. An important implication is that when $\alpha = 0$, the integral of $\varepsilon$ will be fully $[0, 1]$\footnote{The case of $\alpha=1$ ruled out almost all uncertain cases, which is not applicable in differentiable learning.}.
Therefore, our theoretical analysis motivates the goal of calibration as the minimization of the mean squared error between the predicted utility and the observed utility of answers:
\begin{equation}
\mathcal{L} = \sum_{\mathbf{u}(s) > 0} ( U(\phi(s), \hat P_c) - U(\phi(s), \uf ))^2,
\end{equation}
where $ U(\phi(s), \hat P_c)$ represents the predicted utility vector of a soft query $\phi$ according to Definition~\ref{def:utility vector}. SRC with this strategy is denoted as SRC(L).

Notably, we only need to train the calibration transformation in the cases of $\alpha = 0$, achieving a simpler training strategy but better generalization capability when compared to the QE+NE baselines, as will be presented in Appendix~\ref{app:exp-soft-requirements}.

\section{SQUK Dataset Construction}\label{dataset construct}
We provide a brief overview of dataset construction and the details can refer to Appendix~\ref{app:dataset}. 
\subsection{Useful queries and evaluation protocols}
The validation/test uncertain knowledge graph incorporates new facts that will update the utility vectors of specific soft queries. Only these particular queries are considered meaningful and included in the evaluation. The evaluation of soft queries not only considers recall but also accounts for the values of the recalled answers. Therefore, we adopt metrics from the learning-to-rank framework~\cite{liu2009learning} as our evaluation protocol, which includes Mean Average Precision (MAP), Normalized Discounted Cumulative Gain (NDCG), Spearman's rank correlation coefficient ($\rho$), and Kendall's rank correlation coefficient ($\tau$).

We choose 1P, 2P, 2I, 2IN, and 2IL as train query types and add 3IN, INP, IP, 2M, and IM as validation/test query types. The training query types encompass basic operations, allowing us to evaluate the ability of machine learning methods to generalize to commonly used unseen query types~\citep{yin2024rethinking}. We visualize the structure of these types in Figure~\ref{fig: query graphs}. Additionally, Table~\ref{tab:statistics train main} shows the statistics of training queries of our dataset.

\begin{table}[h]
\centering
  \vspace{-1em}

\caption{The statistics of train queries.}
\label{tab:statistics train main}
\resizebox{.8\linewidth}{!}{
\begin{tabular}{crrrrr}
\toprule
KG & 1P & 2P & 2I & 2IN &2IL  \\
\midrule
PPI5k              & 9,724  & 9,750  & 9,754 & 1,500 & 1,500    \\
O*NET20k           & 18,266  & 18,300  & 18,300 & 1,850 &1,850     \\
CN15k             & 52,887  & 52,900  & 52,900 & 5,300 & 5,300     \\
\bottomrule
\end{tabular}
}
\end{table}

\subsection{Uncertain KGs}
We utilize three standard uncertain KGs: CN15k~\citep{chen2019embedding} for encompassing commonsense, PPI5k~\citep{chen2019embedding} for biology, and O*NET20K~\citep{pai2021learning} for employment domains.   These uncertain KGs are noisy and incomplete, requiring the ML models to predict the confidence values.

\subsection{Soft requirements}\label{sec:soft_requirements}
For the $\alpha$ parameter, we establish connections with the percentile value of the relation to represent the necessity value effectively. We assign specific percentiles to different necessity levels: the 25th, 50th, and 75th percentiles correspond to ``low'', ``normal'', and ``high'' necessity criteria, respectively. We ensure that a ``zero'' requirement is assigned when the necessity criteria reaches 0. We also introduce a hybrid strategy that randomly selects necessity values, enabling a comprehensive evaluation.

For the $\beta$ importance setting, we employ two strategies: ``equal'' and ``random''. Under the ``equal'' strategy, all importance values are assigned an importance value of $1.0$. In contrast, the ``random'' strategy introduces variability by assigning random decimal numbers between 0 and 1 to represent the importance of each soft atomic formula. 


\section{Experiments}
In this section, we empirically explore how to answer soft queries. We mainly compare our method with generalized SoTA CQA models on SQUK dataset, including commonly used query embedding models and advanced symbolic search methods. Additionally, we evaluate the performance of advanced commercial LLMs on soft queries with clear natural language descriptions. The implementation details of these experiments are in Appendix~\ref{app: details implement}. We also conduct the ablation study regarding the distribution and impact of two parameters $\alpha$ and $\beta$ on both kinds of approaches, which is presented in Appendix~\ref{app:exp-soft-requirements} due to page limitations. We also provide the qualitative analysis of the experiments in Appendix~\ref{qualitative analysis}.



\begin{table*}
\scriptsize
\centering
\caption{Result of answering soft queries. 
Logic+NE and ConE+NE refer to the query embedding with number embedding extensions. SRC is our inference method, and SRC(D), SRC(L), and SRC(D+L) are explained in Section~\ref{sec:calibration}. The four metrics are all higher, indicating better performance. }
\label{Main result}
\resizebox{1\linewidth}{!}{
\begin{tabular}{@{}cccccccccccccccccc@{}}
\toprule
\multirow{2}{*}{Uncertain KG}      & \multirow{2}{*}{Models} & \multicolumn{13}{c}{$\tau$}                                                                            & \multirow{2}{*}{AVG. $\rho$} & \multirow{2}{*}{AVG. MAP} & \multirow{2}{*}{AVG. NDCG} \\ \cmidrule(lr){3-15}
                          &                     & 1P   & 2P   & 2I   & 2IN  & 2IL  & 2M   & 2U    & 3IN  & IP   & IM   & INP  & UP   & AVG.           &                         &                      &                       \\ \midrule
\multirow{6}{*}{CN15k}    & LogicE+NE              & 9.1  & -1.5 & 4.8  & 6.0  & 18.3 & 5.1  & -14.1 & 3.5  & -2.4 & 6.4  & -0.9 & 9.6  & 4.8           & 5.8                     & 7.0                  & 11.2                  \\
                          & ConE+NE               & 5.3  & 4.3  & 3.5  & 6.3  & 18.4 & 6.9  & 20.5  & 2.9  & 1.8  & 10.4 & 1.7  & 14.9 & 8.1           & 10.0                    & 7.7                  & 13.2                  \\
& SRC                 & 15.0 & 2.4  & -0.0 & 2.1  & 10.7 & 9.2  & 25.5  & -2.0 & -9.0 & 7.9  & -4.4 & 13.0 & 5.9           & 8.9                     & 9.2                  & 15.5                  \\
& SRC(D)           & 16.6&11.8&-0.6&6.9&10.9&11.7&34.4&0.1&5.2&12.5&4.7&24.2&11.5 & 15.1    & 12.9        & 21.8 \\
& SRC(L) & 15.8 & 11.8 & -0.4 & 2.4  & 11.0 & 12.4 & 32.3  & -0.8 & 3.6  & 11.1 & 1.1  & 22.8 & 10.3 & 13.7           & 12.6        & 21.1         \\ 
& SRC(D+L)           & 15.6&13.4&-0.3&5.2&11.2&13.5&36.8&-0.4&8.2&12.2&4.8&28.2&\textbf{12.4} & \textbf{16.2}           & \textbf{13.7}        & \textbf{23.2}         \\ 
                            \midrule
\multirow{6}{*}{PPI5k}    & LogicE+NE              & 20.5 & 22.6 & 17.1 & 10.4 & 24.4 & 20.0 & 30.4  & 9.1  & 12.4 & 14.1 & -2.6 & 32.9 & 14.8          & 20.7                    & 8.0                  & 16.4                  \\
                          & ConE+NE              & 29.2 & 42.5 & 26.4 & 20.7 & 32.6 & 33.9 & 35.9  & 16.6 & 36.6 & 29.4 & 22.5 & 42.2 & 30.7          & 40.8                    & 44.1                 & 49.2                  \\
& SRC                 & 66.6 & 70.9 & 49.9 & 42.7 & 71.0 & 42.9 & 71.6  & 32.7 & 65.6 & 37.1 & 57.2 & 70.4 & 56.5          & 66.7                    & 68.4                 & 70.7                  \\
& SRC(D)           & 66.7&68.7&53.1&44.9&73.1&46.7&73.3&37.7&63.7&41.3&56.9&69.9&58.0& 68.5    & 64.0        & 69.8 \\

& SRC(L)            & 66.8 & 71.7 & 52.7 & 43.4 & 72.8 & 43.8 & 72.7  & 34.4 & 66.6 & 38.3 & 58.0 & 71.4 & 57.7 & 67.8           & \textbf{69.8}        & \textbf{71.6}         \\ 
& SRC(D+L)           & 66.9&69.0&53.5&45.1&73.5&46.9&73.4&38.1&63.8&41.6&57.1&69.9&\textbf{58.2}& \textbf{68.7}    & 64.1       & 70.1 \\ 
                                                      \midrule
\multirow{6}{*}{O*NET20k} & LogicE+NE              & 6.3  & 9.5  & 43.5 & 3.9  & 36.6 & 9.7  & 15.3  & 8.3  & 11.1 & 8.8  & 3.8  & -9.8 & 13.8          & 18.5                    & 3.5                  & 6.4                   \\
                          & ConE+NE                 & 30.8 & 41.9 & 57.0 & 21.8 & 46.0 & 37.7 & 49.7  & 48.0 & 22.7 & 21.7 & 14.2 & 53.1 & 36.8          & 47.2                    & 27.5                 & 38.7                  \\ 
& SRC                 & 72.0 & 54.9 & 68.6 & 67.6 & 67.3 & 36.9 & 76.0  & 59.2 & 47.6 & 29.1 & 48.9 & 52.4 & 57.3          & 65.3                    & 27.1                 & 41.3                  \\
& SRC(D)           & 71.7&55.2&74.3&67.7&70.9&49.3&80.1&65.0&52.7&44.7&48.9&55.4&61.8& 70.2    & 26.6        & 41.5 \\
& SRC(L)             & 71.6 & 56.7 & 69.8 & 66.8 & 68.4 & 38.2 & 77.6  & 59.9 & 51.3 & 32.1 & 49.8 & 55.4 & 58.7 & 66.5           & \textbf{27.6}        & \textbf{41.8}         \\ 
& SRC(D+L)           & 71.7&55.6&74.3&67.5&71.0&49.7&80.2&65.0&52.9&45.2&49.2&55.9&\textbf{61.9}& \textbf{70.5}    & 26.7    & 41.7 \\
                       \bottomrule
\end{tabular}}
\end{table*}

\subsection{Main results}
\textbf{Baselines}
We select two mainstream CQA methods as baselines: query embedding and symbolic search. Specifically, we focus on two classical query embedding methods: LogicE~\citep{luus_logic_2021} and ConE~\citep{zhang_cone_2021}. To enable soft queries, we incorporate the relation projection network with  Number Embedding (NE) and adjust the loss function accordingly. \footnote{Detailed information can be found in Appendix~\ref{app:float_emb}.} The forward inference of our method, SRC, is directly generalized from SoTA symbolic methods FIT~\citep{yin2024rethinking}, which serve as baselines for search methods. 

\noindent\textbf{Models analysis.} 
The main results are presented in Table~\ref{Main result}, demonstrating that our proposed method significantly outperforms both query embedding methods and directly generalized symbolic methods. By leveraging the two calibration strategies, debiasing (D) and learning (L), as explained in Section~\ref{sec:calibration}, our method achieves superior results across most KGs and metrics on average.



\noindent\textbf{Query structure analysis.}
Although ConE performs well on some trained query types, it struggles with newly emerged query types and those involving negation, such as INP and IM. In contrast, our method exhibits excellent performance across the majority of query types, demonstrating robust combinatorial generalization capabilities on complex queries. Our method particularly excels in handling challenging query types that involve existential variables, such as 2P, 2M, IM, and INP, highlighting its advantages in these scenarios.

\subsection{The comparison with LLMs}
\begin{table}[h]
\centering
\caption{The accuracy of manually annotated queries.}
\label{llm eval}
\resizebox{\columnwidth}{!}{
\begin{tabular}{cccccc}
\toprule
Model &Llama3.3 70B & Gemini-1.5-pro & GPT-3.5-turbo &GPT-4-preview  &SRC    \\
\midrule
Accuracy&  42.6  & 37.1 & 34.3 &37.8 &\textbf{48.9} \\    
\bottomrule
\end{tabular}
}
\vspace{-1em}
\end{table}
We devise an evaluation framework to assess the performance of LLMs, benchmarking their powerful reasoning abilities over uncertain knowledge. To ensure fairness of the comparison, we consider the queries sampled from CN15k and we choose four candidate answers for each query. These queries have also been manually filtered and labeled to ensure clearness and correctness. We describe the syntax and semantics of soft queries using natural language, prompting LLMs to select the most suitable answer.  The details on this setting construction can refer to  Appendix~\ref{app:llm setting}.


The results, shown in Table~\ref{llm eval}, indicate that even the simple symbolic \method\ achieves significantly higher accuracy compared to Llama3.3 70B, Gemini-1.5-pro,  GPT-3.5-turbo, and  GPT-4-preview. This demonstrates that large language models (LLMs) struggle with complex arithmetic operations involving uncertain values of knowledge. Our evaluation is fair, as the required uncertain knowledge is derived from well-known commonsense KGs ConceptNet, and the logical operations are expressed in natural language. Nevertheless, even advanced commercial LLMs struggle to select the highest-scoring answer. This further emphasizes the difficulties presented by the proposed soft queries and highlights the ongoing need for the development of symbolic approaches.


\section{Conclusion}
In this paper, we introduce a novel setting, soft queries on uncertain KGs, which further extends complex logical queries on KG. The soft queries consider the incompleteness of large-scale uncertain KGs and require the incorporation of ML methods to estimate scores for new relational linking while handling semiring algebraic structures. Our proposed soft queries also propose the soft requirements inspired by soft constraint satisfaction problems to control the uncertainty of knowledge. To facilitate the research of soft queries, we construct a soft query answering dataset consisting of three uncertain KGs. Furthermore, we propose a new neural-symbolic approach with both forward inference and backward calibration. Both theoretical analysis and experimental results demonstrate that our method has satisfactory performance. 

\section{Limitation}
Soft queries extend complex logical queries over Knowledge Graphs (KGs) by incorporating soft requirements within uncertain KGs. However, the scope of the proposed soft queries is limited, as it primarily focuses on conjunctive queries. While conjunctive queries form the foundation of complex logical queries, this restriction may hinder the expressiveness and applicability of the proposed soft queries. Furthermore, the dataset does not include cyclic queries, which are NP-complete, even though their complexity can be more easily addressed.

\section{Potential Impact}
Soft queries have the potential to perpetuate existing biases present in the underlying knowledge graphs. If these graphs contain skewed or discriminatory information, the results generated by soft queries may reflect and amplify these biases, leading to unfair outcomes in applications such as hiring, credit scoring, or law enforcement. This raises significant ethical concerns about fairness, as marginalized groups may be disproportionately affected by biased query results, resulting in systemic inequality.

The utilization of soft queries to extract information from knowledge graphs can pose serious privacy risks. If queries access sensitive personal data without proper safeguards, there is a potential for unauthorized disclosures that violate individuals' privacy rights. This concern is heightened in contexts where the data might be used for profiling or surveillance, making it imperative to establish robust privacy protections and ethical guidelines to ensure that individuals' information is handled responsibly and transparently.
\bibliography{ref}

\newpage
\appendix

\onecolumn

\begin{center}
    \Huge Appendix
\end{center}

\section{Related Work}
\label{related work}
\subsection{Complex logical queries}

Answering complex logical queries over knowledge graphs is naturally extended from link prediction and aims to handle queries with complex conditions beyond simple link queries. This task gradually grows by extending the scope of complex logical queries, ranging from conjunctive queries~\cite{hamilton_embedding_2018} to Existential Positive First-Order (EPFO) queries~\cite{ren_query2box_2020}, Existential First-Order (EFO) queries~\cite{ren_beta_2020}, real Existential First-Order queries~\cite{yin2024rethinking}. The primary method is query embedding, which maps queries and entities to a low-dimensional space. The form of embedding has been well investigated, such as vectors~\cite{hamilton_embedding_2018,chen_fuzzy_2022,bai_query2particles_2022}, geometric regions~\cite{ren_query2box_2020,zhang_cone_2021}, and probabilistic distributions~\cite{ren_beta_2020,choudhary_probabilistic_2021,yang_gammae_2022,wang_wasserstein-fisher-rao_2023}. These methods not only explore knowledge graphs embedding but also leverage neural logical operators to generate the embedding of complex logical queries.

There are also neural-symbolic models to answer complex logical queries. Gradient optimization techniques were employed to estimate the embedding existential variables \cite{amayuelas_neural_2021, arakelyan_adapting_2023}.   Graph neural network \cite{zhu_neural-symbolic_2022} was adapted to execute relational projects and use logical operations over fuzzy sets to deal with more complex queries.  Efficient search algorithms based on link predictor over knowledge graphs were presented \cite{yin2024rethinking, bai_answering_2023}. While symbolic methods demonstrate good performance and offer interpretability for intermediate variables, they often struggle to scale with larger graphs due to their high computational complexity.

Many other models and datasets are proposed to enable answering queries with good performance and additional features, see the comprehensive survey~\cite{ren_neural_2023}. However, to the best of our knowledge, there is currently no existing query framework specifically designed for uncertain knowledge graphs.

\subsection{Uncertain knowledge graph embedding}

Uncertain knowledge graph embedding methods aim to map entities and relations into low-dimensional space, enabling the prediction of unknown link information along with confidence values. There are two primary research directions in this field.

The first line of research focuses on predicting the confidence score of uncertain relation facts. UKGE \citep{chen2019embedding} was the pioneering effort to model triple plausibility as the activated product of these embedding vectors. UKGE incorporates soft probabilistic logic rules to provide the plausibility of unseen facts. Building upon this, BEUrRE \citep{chen2021probabilistic} utilizes complex geometric boxes with probabilistic semantics to represent entities and achieve better performance.  Semi-supervised learning was applied  \cite{chen2021passleaf}to predict the associated confidence scores of positive and negative samples. And  Graph neural networks were used \cite{DBLP:journals/sigkdd/BrockmannKB22, liang2023gnn} to represent and predict uncertain knowledge graphs.

The other line of research aims to address link prediction on uncertain knowledge graphs by fitting the likelihood of uncertain facts. To adjust the similar task, FocusE~\cite{pai2021learning} was introduced, an additional layer to the knowledge graph embedding architecture. They provide variants of classical embedding methods such as TransE~\citep{bordes_translating_2013}, DistMult~\citep{DBLP:journals/corr/YangYHGD14a}, and ComplEx~\citep{trouillon_complex_2016}.


\section{Ablation study: The impact of soft requirements}\label{app:exp-soft-requirements}
The two parameters,  $\alpha$ and $\beta$ play a crucial role in controlling soft constraints, thus we construct 
settings with varying values. Specifically, we select ``zero''(Z) and ``random''(R) for $\alpha$, and ``equal''(E) and ``random''(R) for $\beta$, as explained in Section~\ref{sec:soft_requirements}. Moreover, we sample 12 query types from O*NET20k KG. The detailed construction is in Appendix~\ref{app:varing ab setting}. For ConE, we train it on each setting and test it across all settings. As for \method, we directly test it on all settings.
\begin{table}
\centering
\caption{The mean NDCG of varing $\alpha$ and $\beta$.
}
\label{varying setting}

\begin{tabular}{ccccccc}
\toprule
\multirow{2}{*}{Model} & \multirow{2}{*}{Train} & \multicolumn{4}{c}{Test} &\multirow{2}{*}{AVG.}    \\
                       &                        & Z+E & Z+R & N+E & N+R & \\
\midrule
\multirow{4}{*}{ConE+NE} &Z+E & \textbf{39.3}      &35.6    & 31.9     &31.5   & 34.6     \\
& Z+R  & 39.0      &\textbf{42.2}  &27.8   &28.7  & 34.4      \\

 &N+E  & 26.6     &23.6     & 46.3     &44.1  & 35.2    \\

 &N+R  & 14.0     &7.5    & 43.5     &42.2 &   26.8   \\
\midrule
SRC&-& 34.3 &37.2     & \textbf{46.7}      & \textbf{45.8}   & \textbf{41.0}      \\
\bottomrule
\end{tabular}

\vspace{-1em}
\end{table}

The results in Table~\ref{varying setting} demonstrate that \method~outperforms ConE trained on four different settings in terms of average scores. In the ``Z+E'' and ``Z+R'' settings, although ConE achieves higher scores when trained and tested within the same setting, its performance considerably declines when generalizing to other settings. Our method \method~exhibits consistent performance across various settings due to its strong generalization by theoretical foundations in Section~\ref{sec:error-analysis}.

\section{Qualitative analysis} 
\label{qualitative analysis}
We select the following query for the LLM to conduct a qualitative analysis:
``
Soft query: (bread, is related to, $f_1$, 0.3, 0.9) $\land$ (rice, is related to, $f_1$, 0.7, 0.3).
Four candidate entities: basic, wheat, white, and food.
''
Among the four candidates, basic and wheat are weakly related to the two soft constraints. While white and food exhibit higher relevance, the constraints related to bread carry greater importance. Consequently, the final answer is bread. However, some LLMs may often select rice due to a tendency to hallucinate the concept of white bread. Such errors highlight the challenges in accurately interpreting soft constraints and the need for more robust mechanisms to handle nuanced queries.
\section{Logical queries on knowledge graphs} 
\label{logical queries}

\begin{definition}[Knowledge graphs]
Let $\entities$ be the set of entities and $\relations$ be the set of relations. A knowledge graph is a set of triples $\kg = \{(s_i, r_i, o_i)\}$, where $s_i,o_i\in \entities$ are entities and $r_i\in \entities$ is relation.
\end{definition}

The fundamental challenge of knowledge graphs lies in dealing with the Open World Assumption (OWA). Unlike the Closed World Assumption (CWA), which considers only observed triples as facts, OWA acknowledges that unobserved triples may also be valid.

The study of logical queries on KG considers the Existential First-Order (EFO) queries, usually with one free variable~\cite{ren_beta_2020,wang_benchmarking_2021,yin2024rethinking}.
\begin{definition}[Syntax of existential first-order queries]\label{def:efo-query}
The disjunctive normal form of an existential first-order query $\Gamma$ is:
\begin{align}\label{eq:efo-query-dnf}
    \Gamma(y) = \gamma_1 (y) \lor \cdots \lor \gamma_q (y),
\end{align}
where $y$ is the variable to be answered. Each $\gamma_i(y)$ is a conjunctive query that is expanded as
\begin{align}
    \gamma_i(y) = \exists x_1, \dots, x_{n}. a_{i1} \land \cdots \land a_{im_i}, i=1,...,q,
\end{align}
where $x_1, \dots, x_{n}$ are existentially quantified variables, each $a_{ij} = r(h, t)$ or $a_{ij} = \lnot r(h, t), j= 1,..., m_i$ is an atomic query, $r$ is the relation, $h$ and $t$ are either an entity in $\entities$ or a variable in $\{y, x_1, ..., x_n\}$.
\end{definition}
\begin{definition}
    $\Gamma(s/y)$ denotes the substitution of the entity $s$ for the variable $y$.
\end{definition}
When all free variables are substituted, \textbf{a query} $\Gamma(y)$ is transformed into \textbf{a sentence} $\Gamma(s/y)$. Given $\kg$, answering a query $\Gamma(y)$ means finding all substitutions, such that the sentence $\Gamma(s/y)$ is entailed by $\kg$, i.e., $\kg \models \Gamma(s/y)$. The answer set is defined as
\begin{definition}[The Answer Set of first-order 1uery]\label{def:efo1-query}
    The answer set of an first-order query is defined by
    \begin{align}
        \mathcal{A}[\Phi(y)] = \{a\in \entities |\text{ }\Phi(a) \text{ is True}\}
    \end{align}
\end{definition}

The answers in the answer set derived from $\kg_{\rm train}$ is easy answers. hard answers are the answers in the set difference between the answers from $\kg_{\text{valid}}$ and $\kg_{\text{train}}$~\cite{wang_benchmarking_2021,ren_neural_2023}. The traditional graph-matching methods can not find the answers introduced by new facts~\citep{riesen2010exact}. Thus, we should develop new methods.

\subsection{ML-based method for logical queries on KG}
Recent works are dedicated to introducing ML methods, i.e., knowledge graph embeddings, to generalize from $\kg$ to $\widehat{\kg}$, so that it approximates $\kg_{\rm test}$.

\noindent\textbf{Query Embeddings (QE).} Query embedding methods generally map the query $\Gamma$ as an embedding~\citep{ren_beta_2020,liu_neural-answering_2021,wang_logical_2023}. One dominant way is to ``translate'' the procedure of solving logical formulas in Equation~\eqref{eq:efo-query-dnf} into the set operations, such as set projection, intersection, union, and complement. Then, the neural networks are designed to model such set operations in the embedding space. We can see that direct modeling of set operations is incompatible with the concept of necessity and importance introduced in Section~\ref{sec:formulation}.

\noindent\textbf{Inference methods.} Other methods solve the open world query answering methods in a two-step approach~\citep{arakelyan_complex_2021,bai_answering_2023,yin2024rethinking}. In the first step, the pre-trained knowledge graph embedding estimates the $\kg_{\rm test}$. Then, the answers are derived by the fuzzy logic inference or optimization. These methods rely on the standard logic calculation and cannot be directly applied to our SQUK setting introduced in Section~\ref{sec:formulation}.




\section{Connection with Open World Assumption}\label{app:connection with owa}

Evaluating queries over deductive question-answering systems generally follows either the closed-world assumption (CWA) or the open-world assumption (OWA)~\cite{reiter1981closed}. Under CWA, only known facts are considered true, whereas OWA assumes that the absence of knowledge does not imply falsity~\cite{reiter1986sound}. Since knowledge graphs (KGs) are often incomplete, knowledge graph completion (KGC) has been proposed to address this challenge.
In KGC, the observed knowledge in KGs is divided into three nested subsets: $\kg_{\text{train}} \subset \kg_{\text{valid}} \subset \kg_{\text{test}}$. The KGC model is trained on the $G_{\text{train}}$ subset and evaluated on the $G_{\text{valid}}$ subset to assess its performance.
For complex logical query answering, the evaluation considers "hard answers," defined as the set difference between the answers from $G_{\text{valid}}$ and $G_{\text{train}}$~\cite{wang_benchmarking_2021,ren_neural_2023}. This approach assesses the model's ability to handle incomplete knowledge and make inferences beyond the observed facts in the training set. The evaluation of complex logical query answering extends beyond CWA but does not fully align with OWA. The same applies to the evaluation of soft query answering.
Evaluating under the Open-World Assumption represents a promising avenue for future work~\cite{yang2022rethinking}.

\section{Constraint Satisfaction Problem and Soft Constraint Satisfaction Problems}\label{app:semiring-CSP}

\textbf{Constraint Satisfaction Problems (CSP)} is a mathematical question defined as a set of objects whose state must satisfy several constraints or limitations. Each instance of it can be represented as a triple $(Z, D, C)$, where $Z=(z_1, \cdots, z_n)$ is a finite tuple of $n$ variables, $D=(D_1, \cdots, D_n)$ is the tuple of the domains corresponding to variables in $Z$, and $C=\{(C^1_1, C^2_1), \cdots, (C^1_t, C^2_t)\}$ is the finite set of $t$ constraints.  $D_i$ is the domain of $z_i$,  and $C^1_i \subset Z$ is the scope of the i-th constraint and $C^2_i$ specifies how the assignments allowed by this constraint. In general definitions, the constraints in classical Constraint Satisfaction Problems are hard, meaning that none of them can be violated.

Many problems can be viewed as CSP, which include workforce scheduling and the toy 8-queens problem. Conjunctive queries are a special case to be reduced as CSP under open-world assumptions if we set the constant variable’s domain as itself, set the domain of the existential variable and free variable as the entity set $\mathcal{E}$ of knowledge graphs, and treat atomic formula or its negation as binary constraint by knowledge graph. 
 
\textbf{Soft Constraint Satisfaction Problems}

\begin{table}[t]
\centering
\caption{Different specific soft CSP frameworks modeled as c-semirings.}
\label{tab: NELL EFO1 result}
\begin{tabular}{cccc}
\toprule
Semiring      & $E$         & $\times_s$ & $+_s$\\
\midrule
Classical     & $\{T,F \} $ & $\land$  & $\lor$ \\
Fuzzy         & $[0,1]$ &  $min$ & $max$ \\
k-weighted    &   $\{0, \dots, k \}$   & $+$ & $min$ \\
Probabilistic &   $[0,1]$  &  $xy$        &  $max$ \\
Valued        & $E$  & $\oplus$ &    $min_v$  \\
\bottomrule
\end{tabular}
\end{table}

Though CSP is a very powerful formulation, it fails when real-life problems need to describe the preference of constraint. It usually returns null answers for problems with many constraints, which are called over-constraints. To tackle the above weakness, many versions of Soft Constraint Satisfaction Problems (SCSP) are developed, such as fuzzy CSP, weighted CSP, and probabilistic CSP, which all follow a common semiring structure, where two semiring operations $\times_s, +_s$ are utilized to model constraint projection and combination respectively. Based on this theoretical background, we propose soft queries based on SCSP. The proposed soft queries will have the advantages of SCSPs and can handle the numeric facts representing uncertainty.
\section{Details of Implementation}

Our experiments are run on the Nvidia V100-32G.
\label{app: details implement}
\subsection{adaptive scoring }
Let $E_s$, $E_r$, and $E_t$ be the embedding vectors of entity $s$, relation $r$, and entity $t$ respectively. We parameterize the adaptive scoring calibration using the following learnable affine transformations:
\begin{align}
    \rho_{\theta}(s, r, o) &= W^1_1 E_s + b^1_1 + W^1_2 E_r + b^1_2 + W^1_3 E_t + b^1_3, \\
    \lambda_{\theta}(s, r, o) &= W^2_1 E_s + b^2_1 + W^2_2 E_r + b^2_2 + W^2_3 E_t + b^2_3,
\end{align}
where $\{W^i_j, b^i_j\}$ for $1 \le i \le 2$ and $1 \le j \le 3$ are the learnable parameters.
\subsection{Uncertain knowledge graph embedding}
We reproduce previous results~\citep{chen2019embedding,chen2021probabilistic} and use the same embedding dimension. We search the other parameters including the learning rate from $\{1e-3, 5e-4, 1e-4 \}$ and regularization term $\lambda$ from $\{0.1, 0.01, 0.05 \}$.
\subsection{Query embedding with number embedding}

We follow the same hyperparameter of origin paper~\citep{luus_logic_2021,zhang_cone_2021} but search the learning rate and margin. The embedding dimension of number embedding is the same as the dimension of entity embedding.
\subsection{Two strategies of calibration}
For learning strategy, we search learning rate from $\{5e-4, 5e-5, 1e-5 \}$. For Debiasing strategy, we search $\epsilon$ from  $\{0.05, 0.1, 0.15 \}$.

\section{Details of Soft Reasoning with Calibrated Confidence Values}
\label{details methods}
\begin{figure}[t]
  \centering
  \includegraphics[width=.95\textwidth]{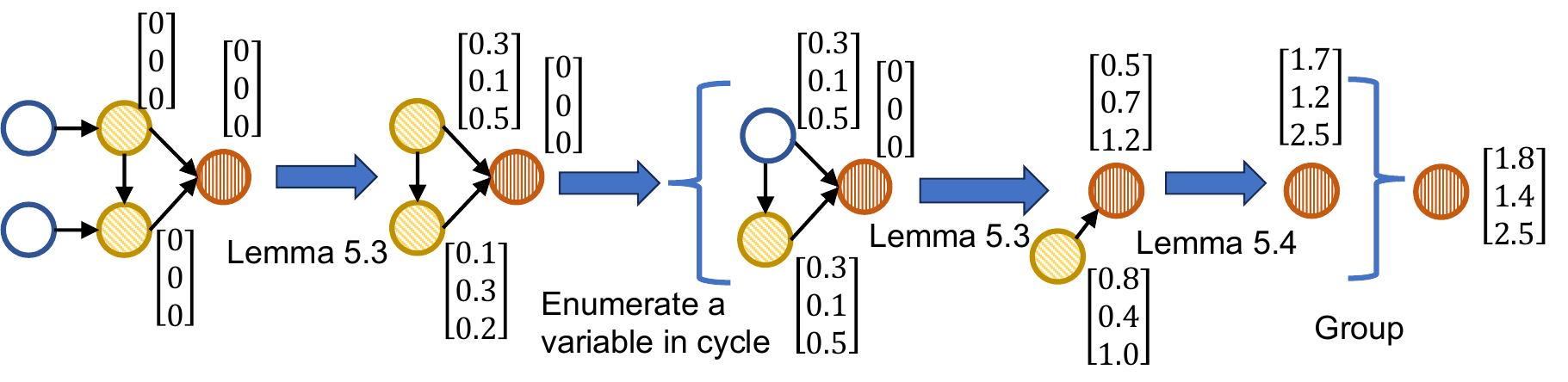}
  \caption{A toy model to present the process of \method~ algorithm.}
  \label{fig: algorithm}
\end{figure}
\subsection{Uncertain value definition}
By indexing all entities and relations, we represent $s$, $r$, and $o$ as integers. Confidence score prediction over uncertain knowledge graphs can be conceptualized by the neural link predictor as $|\mathcal{R}|$ relation matrices $\ukge_r \in \mathbb{R}^{|\mathcal{E}| \times |\mathcal{E}|}$, where $\ukge_r(s, o)$ is the predicted score of fact triple $(s,r,o)$ and $n$ is the number of entities. We denote $U_r(s)$ as a vector formed by the elements of the $s$-th row. The symbol $+$ also denotes element-wise plus operation when used in vector-vector or matrix-matrix operations.
We also define two new plus operations in matrix-vector operations.

\begin{definition}
    Given a matrix $\mathbf{M_1} \in \mathbb{R}^{|\mathcal{E}| \times |\mathcal{E}|}$ and a vector $\mathbf{b} \in \mathbb{R}^{|\mathcal{E}|}$, We define two new addition operations: column-wise addition and row-wise addition as following:
    \begin{align}
        \mathbf{M_2} = \mathbf{M_1} +_r \mathbf{b}, \mathbf{M_2}(s,o) = \mathbf{M_1}(s, o) + \mathbf{b}(s), \\
        \mathbf{M_2} = \mathbf{M_1} +_c \mathbf{b}, \mathbf{M_2}(s,o) = \mathbf{M_1}(s, o) + \mathbf{b}(o).
    \end{align}
\end{definition}
\begin{definition}[Membership function]
    Given a soft query and a variable $x$, $\mu(x, C_x)$ is a membership function to check the current confidence value, where $\hat U(\mu(s/x, C_x)) = C_x(s)$.
\end{definition}

\subsection{Details of the inference of $\text{\method}$}
Since the effect $\alpha$ and $\beta$ are equivalent to modifying the uncertain values, we will focus on explaining how \ the method works in the basic scenario where $\alpha=0.0$ and $\beta=1.0$. Our goal is to infer the utility vector $\hat U[\phi(y)]$ by estimating the confidence value $\hat U[\phi](o) = [\ukge \models_s \phi(o/y)]$, for all $o \in \mathcal{E}$.  In the following content, we will cut the query graph step by step while recording any lost information on the remaining nodes. After removing nodes from the soft query $\phi$, we denote the remaining sub-query graph as $\psi$.

\subsubsection{Step 1. Initialization.\label{step1}} We initialize each variable $x$ as a candidate state vector $C$ with all zero elements, denoted as $C = \mathbf{0} \in \mathbb{R}^{|\mathcal{E}|}$ This vector records the candidates and their corresponding values. Throughout the algorithmic process, the vectors $C$ are updated iteratively, ultimately yielding the final answers represented by the resulting vector $C_y$. 


\subsubsection{Step 2. Remove constant nodes.}\label{step2} For constant nodes in a query graph, we can easily remove their whole edges and update the information to connected nodes by the following lemma. 

\begin{lemma}\label{remove_constant_nodes}
For the constant nodes in a soft query, there exists an $O(|\mathcal{E}|)$ transformation $T_c$ to remove them. 
\end{lemma}
\begin{proof}
Without loss of generality, we consider the situation that a constant node with entity $e$ connects an existential variable $x$ that there is only one positive edge from $e$ to $x$, one negative edge from $e$ to $x$, and one positive edge from $x$ to $e$. To simplify, we also denote $e$ as the related grounded entity.
\begin{align}
    \hat U [\phi](o) &= \hat U(\exists x.\mu(x, C_x)  \wedge r_1(e, x) \wedge \neg r_2(e, x) \wedge r_3(x, e) \wedge \psi(o/y)) \notag \\ 
                  &= \max_{s \in \mathcal{E} } [C_x(s) + \ukge_{r_1}(e, s) + (1 - \ukge_{r_2})(e, s) \\ & +   \ukge_{r_3}^T(e, s) + \hat U(\psi(o/y;s/x))]  \notag\\ 
                  &= \max_{s \in \mathcal{E}} [C_x'(s) + \hat U(\psi(o/y; s/x))], \notag
\end{align}
where $C_x' = C_x +  \ukge_{r_1}(e) +  (\mathbf{1}-\ukge_{r_2})(e) +  \ukge_{r_3}^T(e)$ is updated candidate vector and $\mathbf{1} \in \mathbb{R}^{|\mathcal{E}| \times |\mathcal{E}|}$ is all one matrix. The situation where a constant node connects a free variable node is similar and even easier to handle.
\end{proof}
In the above derivation, we retain the value of answers by updating the candidate state vector of an existential variable. 

\subsubsection{Step 3. Remove self-loop edges.\label{step3}}  

\begin{lemma}\label{remove_self-loop_edges}
For a soft query having self-loop edges, there exists an $O(|\mathcal{E}|)$ transformation $T_c$ to remove self-loop edges.
\end{lemma}

\begin{proof}
Without loss of generality, we consider the situation in which an existential variable $x$ contains one positive loop.
\begin{align}
    \hat U[\phi](o) &= \hat U(\exists x.\mu(x, C_x)  \wedge r(x, x) \wedge \psi(o/y; x)) \notag \\ 
                  &= \max_{s \in \mathcal{E} } [C_x(s) + \ukge_r(s,s) + \hat U(\psi(o/y;s/x))]  \notag\\ 
                  &= \max_{s \in \mathcal{E}} [C_x'(s) + \hat U(\psi(o/y; s/x))], \notag
\end{align}
where $C_x' = C_x +  diag(\ukge_{r})$ and $diag(\ukge_r)$ is a vector formed by the diagonal elements of matrix $\ukge_r$.
\end{proof}

Previous research~\citep{yin2024rethinking} has demonstrated the difficulty of sampling high-quality self-loop queries due to the rarity of self-loop relations in real-life knowledge graphs. Similarly, in our specific uncertain knowledge graph, it is challenging to sample meaningful queries.

\subsubsection{Step 4. Remove leaf nodes.}\label{step4} The leaf node $u$, which is only connected to one other node $v$ in the soft query graph.  After removing the constant nodes, if the query graph contains no circles, we can get the utility vector by removing the leaf nodes step by step. Next, we present how to handle leaf nodes by the three lemmas.

\begin{lemma}\label{remove_leaf_nodes}
If the leaf node $u$ is an existential variable $x$ and $v$ is the free variable $y$, there exists an $O(|\mathcal{E}|^2)$ transformation $T_l$ to shrink its graph. 
\end{lemma}

\begin{align} 
    \hat U[\phi](o) &= \hat U(\exists x.\mu(x, C_x)  \wedge r(x, o) \wedge \mu(o, C_y) \wedge \psi(o/y)) \notag \\ 
                  &= \max_{s \in \mathcal{E}}[C_x(s) + \ukge_{r}(s, o) + C_y(o) + \hat U(\psi(o/y))]  \notag \\
                  &= \max_{s \in \mathcal{E}} [C_x(s) +  \ukge_{r}(s, o) + C_y(o)] + \hat U(\psi(o/y))  \notag \\
                  &= C_y'(o) + \hat U(\psi(o/y)), \notag
\end{align}
where $C_y'(o) = \max_{s \in \mathcal{E}} M_1(s,o)$ and $M_1 = [(\hat P_r +_c C_y) +_r C_x]$.
\begin{lemma}
If the leaf node $u$ is a free variable $y$ and $v$ is the existential variable $x$, we can first solve the subgraph obtained by removing $u$, and then remove $v$.
\end{lemma}

\begin{proof}
Denote $\psi$ is the subgraph obtained by removing $u$, we treat $y$ as the free variable to get its utility vector $\hat U(\psi(y))$. Then we can remove $u$ by the following.
\begin{align}
    \hat U[\phi](o) &= \hat U(\mu(o, C_y)  \wedge [ \exists x. r(x, o) \wedge \psi(x)]) \notag \\ 
                  &= C_y(o) +  \hat U( \exists x. r(x, o) \wedge \psi(x) )  \notag \\
                  &= C_y(o) +  \max_{s \in \mathcal{E}} [ \ukge_{r}(s, o) + \hat U(\psi(o))]  \notag \\
                  &= C_y(o) +  \max_{s \in \mathcal{E}} M_2(s,o),
\end{align}
where $M_2 = \ukge_r +_c \hat U[\psi]$. 
\end{proof}

\begin{lemma}
If the leaf node $u$ and its connected node $v$ both are the existential variable, we can remove the leaf node u when the existential quantifier is maximization.
\end{lemma}

\begin{proof}

It will be difficult when trying to cut the leaf node $x_1$,
\begin{align}
    \hat U[\phi](o) &= \hat U(\exists x_1, x_2. \mu(x_1, C_{x_1})  \wedge  r(x_1, x_2) \wedge \mu(x_2, C_{x_2}) \wedge \psi(o)) \notag \\ 
                  &= \max_{s_1 \in \mathcal{E}, s_2 \in \mathcal{E}} [C_{x_1}(s_1) +  \ukge_r(s_1, s_2) + C_{x_2}(s_2) + \hat U(\psi(o, s_2/x_2))]  \notag \\
                  &= \max_{s_2 \in \mathcal{E}} \max_{s_1 \in \mathcal{E}} [C_{x_1}(s_1) +  \ukge_r(s_1, s_2) + C_{x_2}(s_2) + \hat U(\psi(o, s_2/x_2))].
\end{align}
While the existential quantifier is maximization, it is noteworthy that for any $s_1,s_2, o \in \mathcal{E}$,
\begin{align}
    \max_{s_1 \in \mathcal{E}} [M_3(s_1, s_2) + \hat U(\psi(o; s_2/x_2))]
    &= \max_{s_1 \in \mathcal{E}} [M_3(s_1, s_2)] + \hat U(\psi(o; s_2/x_2)) \notag \\
    &= C_{x_2}'(s_2) + \hat U(\psi(o; s_2/x_2)),
\end{align}
where $M_3 = (\ukge_r +_r C_{x_1}) +_c C_{x_2}$ and $C_{x_2}'(s_2)=\max_{s_1 \in \mathcal{E}} [M_3(s_1, s_2)]$.

Therefore, we can remove $x_1$  by updating $x_2$ as follows,
\begin{equation}
    \hat U[\phi](o) = \max_{s_2 \in \mathcal{E}} [C_{x_2}'(s_2) + \hat U(\psi(o/y; s_2/x_2))].
\end{equation}

\end{proof}
Combining the above three lemmas, we can step by step find a leaf node and remove it when the query graph has no cycles.
\begin{lemma}
If the soft query contains no circles, we can get the utility vector by removing leaf nodes when the existential quantifier is maximization.
\end{lemma}
\subsubsection{Step 5. Enumerate on the cycle.}
To the best of our knowledge, the only precise approach for addressing cyclic queries is performing enumeration over one existential node involved in the cycle, which reads $\hat U(\exists x.\phi(o/y; x)) = \max_{s \in \mathcal{E}} \hat U(\phi(o/y; s/x)).$ Then, we apply Step 4 to remove this fixed existential variable since this variable is equivalent to the constant variable. The query graph breaks this cycle and becomes smaller. The remaining query can be solved by applying Step 4. When solving cyclic queries, the time complexity of this algorithm is exponential.

\subsubsection{Step 6. Getting the utility vector.}
Following the aforementioned steps, the query graph will only contain the free node $y$, resulting in the formula $\mu(y, C_y)$. By definition, the desired utility vector will be $C_y$, which provides the confidence values of all the candidate entities.

\section{Uncertain Knowledge Graph Embeddings}\label{app:UKGE}
We introduce the backbone models for uncertain knowledge graph embedding. The results of changing the backbone are presented in Table~\ref{results of backbone}.

\textbf{UKGE~\cite{chen_learning_2023}} is a vector embedding model designed for uncertain knowledge graphs. It has been tested on three tasks: confidence prediction, relational fact ranking, and relational fact classification. To address the sparsity issue in the graph, UKGE utilizes probabilistic soft logic, allowing for the inclusion of additional unseen relational facts during training.

\textbf{BEUrRE~\cite{chen2021probabilistic}} is a probabilistic box embedding model that has been evaluated on two tasks: confidence prediction and relational fact ranking. This model represents each entity as a box and captures relations between two entities through affine transformations applied to the head and tail entity boxes.

\begin{table*}
\tiny
\centering
\caption{The results of answering complex soft queries with different backbone models. }
\label{results of backbone}
\begin{tabular}{cccccccccccccccc}
\toprule
&Models   &Metrics     & 1P   & 2P   & 2I   & 2IN  & 2IL  & 2M   & 2U  &3IN & IP  & IM   & INP & UP & AVG \\
\midrule

\multicolumn{15}{c}{CN15k} \\
\midrule
&\multirow{4}{*}{SRC (UKGE)}   & MAP&20.6&7.1&7.9&14.4&14.6&3.3&14.7&7.7&6.2&3.2&4.2&6.5&9.2\\
                             && NDCG &27.7&15.4&10.1&23.8&22.6&7.9&24.3&10.4&11.1&6.9&10.5&14.9&15.5\\
                            && $\rho$&21.9&6.4&0.2&5.9&14.3&11.2&32.0&-1.5&-7.0&8.4&-2.3&17.4&8.9 \\
                            &&$\tau$&15.0&2.4&-0.0&2.1&10.7&9.2&25.5&-2.0&-9.0&7.9&-4.4&13.0&5.9\\
\midrule
&\multirow{4}{*}{SRC (BEUrRE)}    & MAP&32.2&11.5&13.2&25.9&29.2&3.9&26.6&13.9&12.8&5.3&8.4&11.6&16.2\\
              & & NDCG&41.5&21.3&15.3&37.4&39.7&9.7&40.5&17.6&19.0&10.4&17.1&22.8&24.3\\
              & & $\rho$    &25.7&13.9&-2.3&11.6&19.0&17.4&33.5&-0.3&-0.4&21.3&5.2&21.5&13.8 \\
              & & $\tau$&18.7&8.7&-3.5&7.7&14.8&14.7&27.2&-1.9&-3.6&20.3&1.8&16.3&10.1\\

\midrule
\multicolumn{15}{c}{PPI5k} \\
\midrule
&\multirow{4}{*}{SRC (BEUrRE)}   & MAP & 78.2&78.4&72.9&67.0&72.0&52.0&73.7&58.6&76.4&54.2&69.2&67.8&68.4\\
              & & NDCG& 80.8&78.2&77.3&68.7&78.0&52.3&79.9&62.9&76.2&52.4&69.4&72.4&70.7\\
               && $\rho$ &77.7&82.3&57.9&50.9&80.6&54.3&83.0&38.8&76.2&47.6&68.3&82.2&66.7\\
             && $\tau$&66.6&70.9&49.9&42.7&71.0&42.9&71.6&32.7&65.6&37.1&57.2&70.4&56.5\\
\midrule
&\multirow{4}{*}{SRC (BEUrRE)}   & MAP&71.0&67.8&63.3&56.3&63.6&51.8&66.3&49.4&67.7&52.6&57.7&58.8&60.5\\
           & & NDCG &74.8&70.5&68.7&61.5&70.0&49.7&74.2&55.5&68.5&49.1&61.8&66.5&64.2\\
           & &$\rho$&72.7&76.4&51.6&47.1&73.9&51.3&77.3&39.0&68.5&43.2&63.2&75.0&61.6\\
           &&$\tau$&59.7&62.3&41.8&37.7&61.9&39.4&64.2&32.0&55.8&32.6&50.6&61.5&50.0 \\
\midrule
\multicolumn{15}{c}{O*NET20k} \\
\midrule

&\multirow{4}{*}{SRC (BEUrRE)} & MAP&24.9&6.4&70.6&26.0&63.3&5.6&32.8&68.5&7.7&7.3&5.2&7.7&27.1\\
               & & NDCG& 44.9&19.5&80.4&46.1&74.0&17.2&57.2&76.3&19.2&17.1&17.7&24.6&41.3\\
               & & $\rho$&77.9&64.8&79.1&73.7&79.6&43.7&83.3&69.4&53.8&33.3&58.0&60.7&65.3\\
               & &$\tau$ &72.0&54.9&68.6&67.6&67.3&36.9&76.0&59.2&47.6&29.1&48.9&52.4&57.3 \\
\midrule
&\multirow{4}{*}{SRC (BEUrRE)}    & MAP&29.0&8.3&65.2&29.6&55.2&7.2&32.3&62.1&9.9&8.8&6.9&9.6&27.1\\
              & & NDCG&53.4&29.2&78.8&52.4&71.0&21.5&60.2&72.7&26.5&20.4&25.5&34.1&45.8\\
              & & $\rho$    &69.0&59.1&78.1&63.6&76.4&40.1&78.3&66.8&47.6&30.4&51.2&57.2&60.2 \\
              & & $\tau$&58.4&47.1&66.6&53.6&62.8&32.0&67.4&55.9&39.2&24.9&40.5&46.8&49.9\\

\bottomrule
\end{tabular}
\end{table*}
\section{Float Embedding for query Embedding}
\label{app:float_emb}
To enable query embedding methods to handle soft requirements in soft queries, we employ floating-point encoding to map floating-point numbers into vectors. These vectors are then added to the relation projection in the query embedding method.

\subsection{Float embedding}

We consider the sinusoidal encoding $g:\mathbf{R} \to R^d$ introduced in Transformer~\citep{vaswani_attenion_2017} and map the values of $\alpha$ and $\beta$ into vector embedding, which can be formulated as:
        \begin{align}
            g(v_i) = \left\{\begin{array}{cc}
            sin(v_i/1000^{i/(2k)})   & i=2k , \\
            cos(v_i/1000^{i/(2k)}) & i=2k+1,
            \end{array}\right.
        \end{align}
where $d$ is the embedding dimension.
\subsection{Modified relation projection}
The query embedding methods usually learn a Multi-Layer Perceptron (MLP) for each relation $r$, which reads as:
\begin{equation}
    S' = \textbf{MLP}_r (S),
\end{equation}
where $S$ is an embedding. Furthermore, the modified relation projection can be expressed as:
\begin{equation}
S' = \textbf{MLP}_r(S + g(\alpha) + g(\beta))
\end{equation}
Here, $g(\alpha)$ and $g(\beta)$ are the embedding of soft requirements $\alpha$ and $\beta$, respectively. By incorporating $g(\alpha)$ and $g(\beta)$ into the relation project net, we enhance the representation of relation projection to better capture the soft requirement.

\section{Details in the Main Dataset Construction}
\label{app:dataset}
\begin{figure*}[t]
  \centering
  \includegraphics[width=1.0\textwidth]{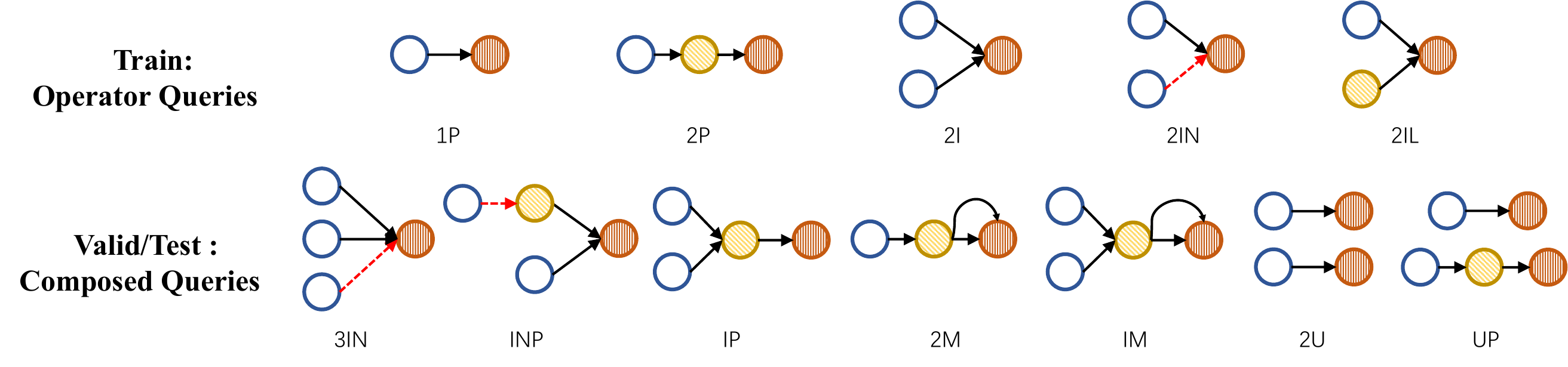}\vskip-1em
  \caption{Query structures of query types. The white, yellow, and red circles represent constant, existential, and free nodes, respectively. The negative atomic formulas are represented by red edges, while atomic formulas are represented by black edges. Like the previous naming convention  \citep{ren_beta_2020, yin2024rethinking}, we use  ''P'' for projection, ``I'' for intersection, ``N'' for negation, ``M'' for multi-edge, and ``L'' for existential leaf.}\vskip-1em
  \label{fig: query graphs}
\end{figure*}
\subsection{Uncertain knowledge graphs}

We sample soft queries from three standard uncertain knowledge graphs\footnote{We leave the exploration of Nl27k~\citep{chen2019embedding} for future work as it involves an inductive setting where the valid/test graphs contain unseen relations.}, covering diverse domains such as common sense knowledge, bioinformatics, and the employment domain.

\textbf{CN15k~\citep{chen2019embedding}} is a subset of the ConceptNet~\citep{robyn2017conceptnet}, a semantic network aimed at comprehending connections between words and phrases.

\textbf{PPI5k~\citep{chen2019embedding}} is a subset of STRING~\citep{szklarczyk2023string}, which illustrates protein-protein association networks collected from  organisms. It assigns probabilities to the intersections among proteins.

\textbf{O*NET20K~\citep{pai2021learning}} is a subset of O*NET, a dataset that describes labeled binary relations between job descriptions and skills. The associated values are to evaluate the importance of the link within the triple.

We present the statistics of these three uncertain knowledge graphs in Table~\ref{tab: Statistics of graph}.
\subsection{Soft requirements}

In the main experiment, we aim for the sampled data to allow the machine learning methods to generalize to various scenarios of soft requirements. Therefore, for  $\alpha$ necessity, we employ a hybrid strategy, randomly selecting from four modes (``zero'', ``low'', ``normal'', ``high'') for each query. As for $\beta$ importance, we utilize a random strategy, randomly choosing a decimal value for different atomic soft formulas in each query.

\subsection{Query types}
The goal of our proposed dataset is to represent the family of existential first-order soft queries systematically. However, including too many query formulas poses challenges in analyzing and evaluating. We select query graphs including operations as train queries and unitize composed query graphs to evaluate the combinatorial generalization.

For training queries, we choose 1P, 2P, 2I, and 2IN, which include soft operators. Additionally, we select 2P for chain queries~\citep{hamilton_embedding_2018}, 2M for multi-edge graphs~\citep{yin2024rethinking}, and 2IL for graphs containing ungrounded anchors~\citep{yin2024rethinking}. More complex soft query graphs can be generated from these basic graphs. We present the statistics of sampled queries in Table~\ref{tab:statistics train main} and Tabel~\ref{tab:statistics valid/test main}.

\begin{table}
\centering
\caption{The statistics of valid/test queries on the main dataset. Different query types have the same number in given uncertain knowledge graph. }
\label{tab:statistics train main}
\begin{tabular}{cccc}
\toprule
KG & CN15k & PPI5k & O*NET20k \\
\midrule
valid             & 3,000  & 2,000  & 2,000     \\
test              & 3,000  & 2,000  & 2,000     \\
\bottomrule
\end{tabular}
\end{table}

\begin{table}
\centering
\caption{The statistics of train queries on the main dataset.}
\label{tab:statistics valid/test main}
\begin{tabular}{cccccc}
\toprule
KG & 1P & 2P & 2I & 2IN &2IL  \\
\midrule
CN15k             & 52887  & 52900  & 52900 & 5300 & 5300     \\
PPI5k              & 9724  & 9750  & 9754 & 1500 & 1500    \\
PPI5k              & 18266  & 18300  & 18300 & 1850 &1850     \\
\bottomrule
\end{tabular}
\end{table}


\subsection{Evaluation protocol}
\label{evaluation protocol}
The open world assumption in uncertain knowledge graphs not only establishes new links between entities but can also potentially refine the values of existing triples. As more observed facts become available, the answers to soft queries not only increase in number but also undergo modifications in terms of their priority. Therefore, to evaluate the relevance judgment, we select several popular metrics commonly used in information retrieval \citep{liu2009learning}, including MAP, DCG, NDCG, and Kendall’s tau.

For each $q$, we denote the set of answers as $\mathcal{A}$, where $a_i \in \mathcal{A}$ represents the i-th answer based on its score, and $r(a), a \in \mathcal{A}$ denotes the predicted ranking of answer $a \in \mathcal{A}$. Our objective is to focus on the precision of answers and the associated predicted ranking information.

\textbf{Mean Average Precision (MAP)}: To define MAP, we first introduce Precision at a given position, defined as:
\begin{equation}
\textbf{P@k}(q) = |\{a \in \mathcal{A} | r(a) \geq k\}| / k.
\end{equation}
Then, Average Precision is defined as follows:
\begin{equation}
\textbf{AP}(q) = (\sum_{k=1}^{|\mathcal{A}|} \textbf{P@k}(q) \cdot l_k) / |\mathcal{A}|,
\end{equation}
where $l_k$ is a binary judgment indicating the relevance of the answer at the kth position. Mean Average Precision is the average AP value across all test queries.

\textbf{Discounted Cumulative Gain (DCG)}:
To calculate the DCG, we utilize the Reciprocal Rank as a relative score for the answers:
\begin{equation}
R(a_i) = 1/r(a_i).
\end{equation}
To incorporate the ranking position, we introduce an explicit position discount factor $\eta_i$. The DCG is then computed as:
\begin{equation}
\textbf{DCG@k}(q) = \sum_{i=1}^k R(a_i) \eta(i),
\end{equation}
where $\eta(i)$ is commonly expressed as $\eta(i) = 1 / \log_2(i+1)$.

\textbf{Normalized Discounted Cumulative Gain (NDCG)}: By normalizing the ideal Discounted Cumulative Gain denoted as $Z_k$, we obtain NDCG:
\begin{equation}
\textbf{NDCG@k}(q) = \textbf{DCG@k}(q) / Z_k.
\end{equation}
\textbf{Kendall's tau}: Kendall's tau is a statistical measure that quantifies the correspondence between two rankings. Values close to $1$ indicate strong agreement, while values close to $-1$ indicate strong disagreement.

\section{Details in Varying Soft Requirements Setting}
\label{app:varing ab setting}
In this setting, we aim to test the model's generalization ability on soft requirements under different strategies. We have chosen a pair of different strategies for each of the two parameters, resulting in a total of four groups. Specifically, we selected "zero" and "normal" for parameter "a," and "equal" and "normal" for parameter "b." Detailed statistics for this setting can be found in the table below. For each group, we follow the procedure and sample train, valid, and test queries. We present the additional results in Table~\ref{diversity result}.
\begin{table}[t]
\centering
\caption{The additional results of varying soft requirements}
\label{diversity result}
\begin{tabular}{ccccccccc}
\toprule
Mode& \multicolumn{2}{c}{Z+E} & \multicolumn{2}{c}{Z+R} & \multicolumn{2}{c}{N+E} & \multicolumn{2}{c}{N+R} \\
Metric& NDCG & $\tau$    & NDCG & $\tau$  & NDCG & $\tau$  & NDCG & $\tau$ \\
\midrule

ConE Z+E & 39.3&8.0      &35.6 & 11.6    & 31.9&30.4      &31.5 & 28.9       \\
\midrule
ConE Z+R  & 39.0&6.7      &42.2 & 14.9  &27.8&29.0     &28.7 & 30.3        \\
\midrule
ConE N+E  & 26.6&32.8      &23.6 & 27.4    & 46.3&38.7      &44.1 & 36.7        \\
\midrule
ConE N+R  & 14.0&24.7      &7.5 & 14.4    & 43.5&38.7      &42.2&38.6        \\
\midrule
\method & 34.3 & 47.8  &37.2 & 41.5     & 46.7& 51.8      & 45.8 & 49.9        \\

\bottomrule
\end{tabular}
\end{table}
\section{Details of Large Language Model Evaluation Setting}

\label{app:llm setting}
Since the entities and relations in CN15k are English words and phrases, the queries sampled from CN15k can be well understood by LLM. In our evaluation, we manually marked and removed meaningless queries. To facilitate our testing process, we selected only four candidate entities for each query. These four candidate entities have large distinctions in terms of scores. To avoid unexpected situations, we manually checked all chosen queries and confirmed that their correct answers could be selected without ambiguity. We present the total numbers and types of selected queries in Table~\ref{tab: statistic llm}.

We articulate the syntax and semantics of our proposed soft queries by using clear natural language. The provided prompts are presented in the subsequent subsection. Through combining queries with prompts, we enable LLM to select the most appropriate answer.
\begin{table}
    \centering
    \caption{The number of annotated queries.}
    \label{tab: statistic llm}
    \begin{tabular}{ccccccccccc}
    \toprule
        Query type & 1P & 2P & 2I & 2IN & 2IL & 2M & 3IN & IP & IM &SUM\\
    \midrule
        Number  & 100 & 50 & 100 & 50 & 15 & 10 & 15 & 15 & 15 &370 \\
    \bottomrule
    \end{tabular}
\end{table}
\subsection{Prompt}
\label{prompt}

\begin{verbatim}

# Background
Given a soft query containing a free variable f, there are four candidate entities for 
this variable and you need to choose the best of them to satisfy the soft query most. In 
order to do so, you can first compute the confidence value to see whether the chosen 
candidate entity satisfies the soft query after substituting the free variable with a 
given candidate entity. After computing the confidence values for each corresponding 
candidate entity, you need to pick the entity that leads to the highest confidence value. 
The detailed steps are described below:

## Definition of Soft Atomic Constraint:
A soft atomic constraint c, a.k.a. a soft atomic formula or its negation, is in the form 
of (h, r, t, \alpha, \beta) or \\neg (h, r, t, \alpha, \beta).

### Notation Description:
In each constraint, we have four different types of variables.
1. r is a relation; 
2. h and t are two terms. Each term represents an entity or a variable whose values 
are entities. And free variable is a term.
3. \alpha is called the necessity value, which represents the minimal requirement 
of the uncertainty degree of this constraint. It can be any decimal between 0.0 and 
1.0. If the confidence value is less than the necessity value, the constraint is 
not satisfied, and thus the final confidence value becomes negative infinity; 
4. \beta represents the priority of this constraint and can be any decimal 
between 0.0 and 1.0.

### Confidence Value of Soft Constraints V(c):
1. The triple (h,r,t) comes from the relation fact in the knowledge graph. In our 
setting, the relation fact r(h,t) is not boolean, and it has a confidence value 
in the range [0.0,1.0] according to its plausibility. When the constraint is 
negative, you need to first estimate r(h,t)—the confidence value of r(h,t)—and
use 1-r(h,t) as the final confidence value for this negative constraint.   
2. \alpha is the threshold in our filter function, working as follows:  
f(v, \alpha) = v,           if  v \geq \alpha,
                    	         -\infty,    if  v < \alpha.
3. \beta is a coefficient.

Thus, the final equation becomes:
V(h, r, t, \alpha, \beta) = \beta \times f(r(h,t), \alpha),
V(\\neg (h, r, t, \alpha, \beta)) = \beta \times f(1-r(h,t), \alpha).

## Definition of Conjunctive Queries \phi:
Soft conjunctive queries are composed of soft constraints.

### Notation Description:
Conjunctive query \phi = c_1 \land \dots \land c_n, where c_i is a soft constraint.

### Confidence Value of Soft Conjunctive Queries V(\phi) :
First, you need to compute the confidence values of all soft constraints in 
this soft conjunctive query. Then, you can simply sum up these confidence 
values as the final confidence value of the soft conjunctive query as follows:
V(\phi) = \sum_i V(c_i).

If the conjunctive query has an existential variable e, you should find an 
entity to replace it and then compute the confidence value of this query.

# Output Format
Please output your response in the JSON format, where the first element 
is the best candidate entity among the four options, and the second element 
is your explanation for your choice.

# Question 
Soft query:(h_1,r_1,f,\alpha_1,\beta_1) \land (h_2,r_2,f,\alpha_2,\beta_2) 
Four candidate entities: s_1, s_2, s_3, s_4

Please return the best candidate for f1 to satisfy the above soft query.

\end{verbatim}

\section{Proof for Error Analysis}\label{app: error analysis}

\subsection{Proof of Theorem~\ref{thm:1p-analysis}}
Firstly, we give the proof of Theorem~\ref{thm:1p-analysis}:

\begin{proof}
    Consider the atomic query $\psi = (a, (\alpha, \beta))$.
    Firstly, we consider whether the soft atomic query is positive or negative. 
    
    Let us assume $a =  r(y,o)$, with $y$ be the only free variable

    Then for arbitrary $r,h,t,\alpha,\beta$:

    \begin{align}
        \Pr\left(\|\hat{U}[\psi](s) - \mathcal{U}[\psi](s)\| >  \delta) \right) & =  \Pr\left(\beta\|[\hat{P}(s,r,o)]_{\alpha} - [\uf(s,r,o)]_{\alpha}\|  > \delta\right)\\
       & =  \Pr\left(\|[\hat{P}(s,r,o)]_{\alpha} - [\uf(s,r,o)]_{\alpha}\|  > \frac{\delta}{\beta}\right) 
    \end{align}

    We note that even if $a = \lnot r(y,o)$, the result is the same:

    \begin{align*}
        \Pr\left(\|\hat{U}[\psi](s) - \mathcal{U}[\psi](s)\| >  \delta) \right) &=  \Pr\left(\beta\|[1-\hat{P}(s,r,o)]_{\alpha} + [\uf(s,r,o)]_{\alpha} - 1\|  > \delta\right) \\
        &=  \Pr\left(\|[\hat{P}(s,r,o)]_{\alpha} -[\uf(s,r,o)]_{\alpha}\|  > \frac{\delta}{\beta}\right) 
    \end{align*}

    For convenience, we write $\hat{x},x$ as the abbreviation of $\hat{P}(s,r,o),\uf(s,r,o)$, correspondingly, then the initial formula becomes:

    \begin{align*}
    &\Pr\left(\|\hat{x} - x\|  > \frac{\delta}{\beta}\right)  \\
   & =    \Pr\left(\|\hat{x} - x\|  > \frac{\delta}{\beta} \mid \hat{x}>\alpha,x>\alpha \right)\Pr(\hat{x},x>\alpha) + \Pr\left(\|\hat{x}\|  > \frac{\delta}{\beta} \mid \hat{x}>\alpha,x<\alpha \right)\Pr\left(\hat{x}>\alpha,x<\alpha\right)  \\
   & +
        \Pr\left(\|x\|  > \frac{\delta}{\beta} \mid \hat{x}<\alpha,x>\alpha  \right)\Pr(\hat{x}<\alpha,x>\alpha) +
        \Pr\left(\|0\|  > \frac{\delta}{\beta} \mid \hat{x}<\alpha,x<\alpha \right)\Pr(\hat{x}<\alpha,x<\alpha) \\
        &\leq \varepsilon(\frac{\delta}{\beta})\Pr(\hat{x}>\alpha,x>\alpha) + \Pr(\hat{x}>\alpha,x<\alpha) + \Pr(\hat{x}<\alpha,x>\alpha) \\
        &= \varepsilon(\frac{\delta}{\beta}) + (1-\varepsilon(\frac{\delta}{\beta}))[\Pr(\hat{x}>\alpha,x<\alpha) + \Pr(\hat{x}<\alpha,x>\alpha)] - \varepsilon(\frac{\delta}{\beta})\Pr(\hat{x}<\alpha,x<\alpha) \\ 
        &\leq \varepsilon(\frac{\delta}{\beta}) + (1-\varepsilon(\frac{\delta}{\beta}))[\Pr(\hat{x}>\alpha,x<\alpha) + \Pr(\hat{x}<\alpha,x>\alpha)]
    \end{align*}

    Moreover, we use the Total Probability Theorem once again and assume $f(\xi)$ as the probability density function of $x$:

    \begin{align*}
        &   \Pr(\hat{x}>\alpha,x<\alpha) + \Pr(\hat{x}<\alpha,x>\alpha) \\
        &= \int_{0}^{1}\Pr(\hat{x}>\alpha \mid x=\xi<\alpha)f(\xi){\rm d}\xi +
        \int_{0}^{1}\Pr(\hat{x}<\alpha \mid x=\xi>\alpha)f(\xi){\rm d}\xi \\ 
        &= \int_{0}^{\alpha}\Pr(\hat{x}>\alpha \mid x=\xi)f(\xi){\rm d}\xi +
        \int_{\alpha}^{1}\Pr(\hat{x}<\alpha \mid x=\xi)f(\xi){\rm d}\xi
    \end{align*}

    By noting that if $\hat{x}>\alpha$ while $x=\xi<\alpha$, it must have $|\hat{x}-x|>\alpha-\xi$, we know that:

    $$\Pr(\hat{x}<\alpha \mid x=\xi>\alpha)  \leq \Pr(|\hat{x}-x|>\alpha-\xi)$$

    Therefore 

    \begin{align*}
        &\int_{0}^{\alpha}\Pr(\hat{x}>\alpha \mid x=\xi)f(\xi){\rm d}\xi +
        \int_{\alpha}^{1}\Pr(\hat{x}<\alpha \mid x=\xi)f(\xi){\rm d}\xi \\
        &\leq \int_{0}^{\alpha}\Pr(|\hat{x}-x|>\alpha-\xi)f(\xi){\rm d}\xi +
        \int_{\alpha}^{1}\Pr(|\hat{x}-x|>\xi-\alpha)f(\xi){\rm d}\xi \\
        &
        \leq \int_{0}^{\alpha}\varepsilon(\alpha-\xi)f(\xi){\rm d}\xi +
        \int_{\alpha}^{1}\varepsilon(\xi-\alpha)f(\xi){\rm d}\xi \\
        &= \int_{0}^{1}\varepsilon(|\alpha-\xi|)f(\xi){\rm d}\xi
    \end{align*}

    Therefore we finish the proof.





\end{proof}

\subsection{Proof of Theorem~\ref{thm:numerical-stability}}
Then for Theorem~\ref{thm:numerical-stability}, consider the query $\phi = \exists x_1, ..., x_n. \psi_{1} \oland \cdots \oland \psi_{m}$, where $\psi_i = (a_i,(\alpha_i,\beta_i))$, 
the final error should be no more than a linear combination of each soft atomic query:
\begin{proof}
    \begin{align*}
        &\|\hat{U}[\phi](s) - \mathcal{U}[\phi](s)\| \\
        &= \|\max_{x_1=s_1,\cdots,x_n=s_n}(\hat{U}[\psi_{1}](s) + \cdots + \hat{U}[\psi_{m}](s)) - \max_{x_1=s_1,\cdots,x_n=s_n}(\mathcal{U}[\psi_{1}](s) + \cdots + \mathcal{U}[\psi_{m}](s))\|  \\
        &\leq \|\max_{x_1=s_1,\cdots,x_n=s_n} \left([\hat{U}[\psi_{1}](s) - \mathcal{U}[\psi_{1}](s)]  + \cdots + [\hat{U}[\psi_{m}](s) - \mathcal{U}[\psi_{m}](s)]\right) \| \\
        &\leq \max_{x_1=s_1,\cdots,x_n=s_n}\left(\|\hat{U}[\psi_{1}](s) - \mathcal{U}[\psi_{1}](s)\|  + \cdots + \|\hat{U}[\psi_{m}](s) - \mathcal{U}[\psi_{m}](s) \|\right) \\
        &\leq \Sigma_{i=1}^{m} \epsilon(\alpha_i,\beta_i)
    \end{align*}

    The final line relies on the definition of $\epsilon$, which gives the upper bound of error that only depends on $\alpha,\beta$.
\end{proof}









\end{document}